%% file: main.tex
\documentclass{article}
\usepackage{graphicx}

\usepackage[utf8]{inputenc} 
\usepackage[T1]{fontenc}    
\usepackage{hyperref}       
\usepackage{url}            
\usepackage{booktabs}       
\usepackage{amssymb}
\usepackage{nicefrac}       
\usepackage{scrextend}      
\usepackage{subcaption}
\usepackage{enumitem}
\usepackage{algorithm}
\usepackage{algorithmic}
\usepackage{amsthm}
\usepackage{fullpage}
\usepackage[numbers,compress]{natbib}
\usepackage{preamble}
\usepackage[capitalise]{cleveref}

\crefname{theorem}{Th.}{Ths.}
\crefname{algorithm}{Alg.}{Algs.}
\Crefname{algorithm}{Algorithm}{Algorithms}
\crefname{lemma}{Lemma}{Lemmas}
\crefname{proposition}{Prop.}{Props.}
\crefname{assumption}{Assumption}{Assumptions}
\crefname{section}{Sec.}{Sections}
\Crefname{section}{Section}{Sections}

\usepackage{xcolor}
\definecolor{linkcolor}{RGB}{83,83,182}
\definecolor{citecolor}{RGB}{128,0,128}
\hypersetup{
    colorlinks=true,
    citecolor=citecolor,
    linkcolor=linkcolor,
    urlcolor=linkcolor
}

\graphicspath{{figures/}}

\title{Score-Based Change Detection for \\ Gradient-Based Learning Machines}
\author{Lang Liu$^1$ \qquad Joseph Salmon$^2$ \qquad Zaid Harchaoui$^1$ \vspace{0.3cm} \\ $^1$ Department of Statistics, University of Washington, Seattle \\ $^2$ IMAG, University of Montpellier, CNRS, Montpellier}
\date{}

\begin{document}

\maketitle
\begin{abstract}
    The widespread use of machine learning algorithms calls for automatic change detection algorithms to monitor their behavior over time. As a machine learning algorithm learns from a continuous, possibly evolving, stream of data, it is desirable and often critical to supplement it with a companion change detection algorithm to facilitate its monitoring and control. We present a generic score-based change detection method that can detect a change in any number of components of a machine learning model trained via empirical risk minimization. This proposed statistical hypothesis test can be readily implemented for such models designed within a differentiable programming framework. We establish the consistency of the hypothesis test and show how to calibrate it to achieve a prescribed false alarm rate. We illustrate the versatility of the approach on synthetic and real data.
\end{abstract}

\section{Introduction}
\label{sec:introduction}

\input{sections/sec1.tex}

\section{Score-Based Change Detection}
\label{sec:score-based_change_detection}

\input{sections/sec2.tex}

\section{Level and Power}
\label{sec:level_and_power}

\input{sections/sec3.tex}

\section{Experiments}
\label{sec:experiments}

\input{sections/sec4.tex}

\section{Conclusion}
\label{sec:conclusion}

We introduced a change monitoring method called \emph{auto-test} that is well suited to machine learning models implemented within a differentiable programming framework.
The experimental results show that the calibration of the test statistic based on our theoretical arguments brings about change detection test that can capture small jumps in the parameters of various machine learning models in a wide range of statistical regimes.
The extension of this approach to penalized maximum likelihood or regularized empirical risk estimation in a high dimensional setting is an interesting venue for future work.

\section*{Acknowledgments}
This work was supported by NSF DMS 2023166, DMS 1839371, DMS 1810975, CCF 2019844, CCF-1740551, CIFAR-LMB, and research awards.

\clearpage

\bibliographystyle{abbrvnat}
\bibliography{refs}

\clearpage

\appendix

\section*{Outline of Appendix}
The outline of the appendix is as follows.
\Cref{sec:implementation} discusses implementation details of the proposed test and its complexity analysis.
\Cref{sec:proofs} is devoted to prove the level and power consistency of the \emph{auto-test}.
\Cref{sec:experiment} provides addition experiment results on a times series model and a hidden Markov model.

\section{Implementation Details}
\label{sec:implementation}

\input{sections/implement.tex}

\section{Theoretical Results and Proofs}
\label{sec:proofs}

\input{sections/proofs.tex}

\section{Additional Experimental Results}
\label{sec:experiment}

\input{sections/experiments.tex}

\end{document}

%% file: sections/sec1.tex
Statistical machine learning models are fostering progress in numerous technological applications, \eg visual object recognition and language processing, as well as in many scientific domains, \eg genomics and neuroscience.
This progress has been fueled recently by statistical machine learning libraries designed within a differentiable programming framework such as PyTorch~\cite{paszke2019pytorch} and TensorFlow~\cite{Ten15}.

Gradient-based optimization algorithms such as accelerated batch gradient methods are then well adapted to this framework, opening up the possibility of gradient-based training of machine learning models from a continuous stream of data.
As a learning system learns from a continuous, possibly evolving, data stream, it is desirable to supplement it with tools facilitating its monitoring in order to prevent the learned model from experiencing abnormal changes.

Recent remarkable failures of intelligent learning systems such as Microsoft's chatbot \cite{metz2018microsoft} and Uber's self-driving car \cite{knight2018selfdriving} show the importance of such tools. In the former case, the initially learned language model quickly changed to an undesirable one, as it was being fed data through interactions with users. The addition of an automatic monitoring tool can potentially prevent a debacle by triggering an early alarm, drawing the attention of its designers and engineers to an abnormal change of a language model.

To keep up with modern learning machines, the monitoring of machine learning models should be automatic and effortless in the same way that the training of these models is now automatic and effortless.
Humans monitoring machines should have at hand automatic monitoring tools to scrutinize a learned model as it evolves over time. Recent research in this area is relatively limited.

We introduce a generic change monitoring method called \emph{auto-test} based on statistical decision theory.
This approach is aligned with machine learning libraries developed in a differentiable programming framework, allowing us to seamlessly apply it to a large class of models implemented in such frameworks.
Moreover, this method is equipped with a \emph{scanning} procedure, enabling it to detect \emph{small jumps} occurring on an unknown subset of model parameters.
The proofs and more details can be found in Appendix.
The code is publicly available at \emph{github.com/langliu95/autodetect}.

\textbf{Previous work on change detection.}~
Change detection is a classical topic in statistics and signal processing; see~\cite{basseville1993detection,tartakovsky2014sequential} for a survey.
It has been considered either in the offline setting, where we test the null hypothesis with a prescribed false alarm rate, or in the online setting, where we detect a change as quickly as possible.
Depending on the type of change, the change detection problem can be classified into two main categories: change in the model parameters~\cite{hinkley1970inference,deshayes1986offline} and change in the distribution of data streams~\cite{lorden1971procedures,kifer2004detecting,cunningham2012gaussian}.
We focus on testing the presence of a change in the model parameters.

Test statistics for detecting changes in model parameters are usually designed on a case-by-case basis; see~\cite{basseville1993detection,carlstein1994change,zhang1994early,csorgo1997limit,enikeeva2019high} and references therein. These methods are usually based on (possibly generalized) likelihood ratios or on residuals and therefore not amenable to differentiable programming. Furthermore, these methods are limited to \emph{large jumps}, \ie changes occurring simultaneously on all model parameters, in contrast to ours.

\begin{figure}[t]
  \centering
  \includegraphics[width=0.7\textwidth]{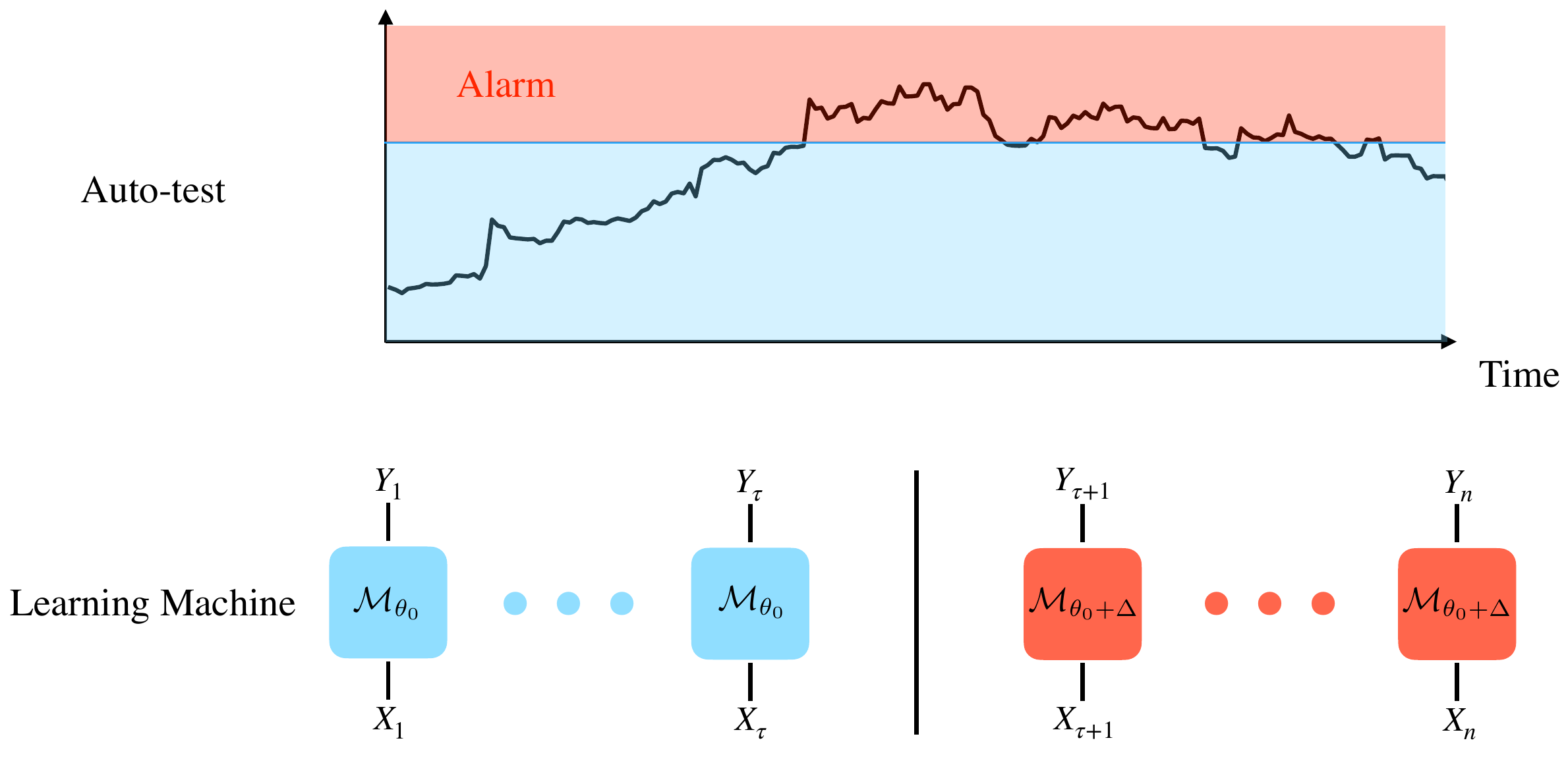}
  \caption{Illustration of monitoring a learning machine with \emph{auto-test}.}
  \label{fig:monitoring}
\end{figure}

%% file: sections/sec2.tex
Let $\obs_{1:n} := \{\obs_k\}_{k=1}^n$ be a sequence of observations.
Consider a family of machine learning models $\{\model_\theta: \theta \in \Theta \subset \bbR^d\}$ such that $\obs_k = \model_\theta(\obs_{1:k-1}) + \varepsilon_k$,
where $\{\varepsilon_k\}_{k=1}^n$ are independent and identically distributed (\emph{i.i.d.}) random noises.
To learn this model from data, we choose a loss function $L$ and estimate model parameters by solving the problem:
\[
  \est_n := \argmin_{\theta \in \Theta} \frac1n \sum_{k=1}^n L\big(\obs_k, \model_\theta(\obs_{1:k-1})\big) \enspace.
\]
This encompasses constrained empirical risk minimization (ERM) and constrained maximum likelihood estimation (MLE).
For simplicity, we assume the model is \emph{correctly specified}, \ie there exists a true value $\tpar \in \Theta$ from which the data are generated.

Under abnormal circumstances, this true value may not remain the same for all observations.
Hence, we allow a potential parameter change in the model, that is, $\theta = \theta_k$ may evolve over time:
\begin{align*}
  \obs_k = \model_{\theta_k}(\obs_{1:k-1}) + \varepsilon_k \enspace.
\end{align*}
A time point $\tau \in [n-1] := \{1,\dots,n-1\}$ is called a \emph{changepoint} if there exists $\Delta \neq 0$ such that $\theta_k = \tpar$ for $k \le \tau$ and $\theta_k = \tpar + \Delta$ for $k > \tau$.
We say that there is a jump (or change) in the data sequence if such a changepoint exists.
We aim to determine if there exists a jump in this sequence, which we formalize as a hypothesis testing problem.
\vspace{0.2cm}
\begin{enumerate}[label=(P\arabic*),start=0,itemsep=1pt,parsep=0pt,topsep=0pt,partopsep=0pt]
  \item \label{prob:P0} Testing the presence of a jump
\end{enumerate}
\begin{align*}
  \hnull: &\ \theta_k = \theta_0 \mbox{ for all } k = 1,\dots,n \\
  \halt:&\ \mbox{after some time $\tau$, $\theta_k$ jumps from $\theta_0$ to $\theta_0 + \Delta$} \enspace.
\end{align*}

We focus on models whose loss $L(\obs_k, \model_\theta(\obs_{1:k-1}))$ can be written as $-\log{p_\theta(\obs_k|\obs_{1:k-1})}$ for some conditional probability density $p_\theta$.
For instance, the squared loss function is associated with the negative log-likelihood of a Gaussian density; for more examples, see, \eg \cite{murphy2012machine}.
In the remainder of the paper, we will work with this probabilistic formulation for convenience, and we refer to the corresponding loss as the probabilistic loss.

\noindent
\textbf{Remark.}
  Discriminative models can also fit into this framework.
  Let $\{(X_i,Y_i)\}_{i=1}^n$ be \iid observations, then the loss function reads $L(Y_k, \model_\theta(X_k))$.
  If, in addition, $L$ is a probabilistic loss, then the associated conditional probability density is $p_{\theta}(Y_k|X_k)$.

\subsection{Likelihood score and score-based testing}
Let $\ind{\{\cdot\}}$ be the indicator function.
Given $\tau \in [n-1]$ and $1 \le s \le t \le n$,
we define the conditional log-likelihood under the alternative as
\begin{align*}
  \ell_{s:t}(\theta, \Delta; \tau) := \sum_{k=s}^t \log{p_{\theta + \Delta\ind{\{k > \tau\}}}(\obs_k|\obs_{1:k-1})} \enspace.
\end{align*}
We will write $\ell_{s:t}(\theta, \Delta)$ for short if there is no confusion.
Under the null, we denote by $\ell_{s:t}(\theta) := \ell_{s:t}(\theta, 0; n)$ the conditional log-likelihood.
The \emph{score function} \wrt $\theta$ is defined as
$S_{s:t}(\theta) := \nabla_{\theta} \ell_{s:t}(\theta)$,
and the \emph{observed Fisher information} \wrt $\theta$ is denoted by
$\info_{s:t}(\theta) := -\nabla_{\theta}^2 \ell_{s:t}(\theta)$.

Given a hypothesis testing problem, the first step is to propose a \emph{test statistic} $R_n$ such that the larger $R_n$ is, the less likely the null hypothesis is true.
Then, for a prescribed \emph{significance level} $\alpha \in (0,1)$, we calibrate this test statistic by a threshold $r_0 := r_0(\alpha)$, leading to a test $\ind\{R_n > r_0\}$, \ie we reject the null if $R_n > r_0$.
The threshold is chosen such that the \emph{false alarm rate} or \emph{type I error rate} is asymptotically controlled by $\alpha$, \ie $\limsup_{n \rightarrow \infty} \prob(R_n > r_0 \mid \hnull) \le \alpha$.
We say that such a test is \emph{consistent in level}.
Moreover, we want the \emph{detection power}, \ie the conditional probability of rejecting the null given that it is false, to converge to $1$ as $n$ goes to infinity.
And we say such a test is \emph{consistent in power}.

Let us follow this procedure to design a test for Problem \ref{prob:P0}.
We start with the case when the changepoint $\tau$ is fixed.
A standard choice is the \emph{generalized score statistic} given by
\begin{align}\label{eq:score_stat}
  R_n(\tau) := S_{\post{\tau+1}}^\top(\htheta_n) \info_n(\htheta_n; \tau)^{-1} S_{\post{\tau+1}}(\htheta_n)\enspace,
\end{align}
where $\info_n(\htheta_n; \tau)$ is the \emph{partial observed information} \wrt $\Delta$ \cite[Chapter 2.9]{wakefield2013bayesian} defined as
\begin{align}\label{eq:schur_complement}
  \info_{\post{\tau+1}}(\hat{\theta}_n) - \info_{\post{\tau+1}}(\hat{\theta}_n)^\top\info_{1:n}(\hat{\theta}_n)^{-1} \info_{\post{\tau+1}}(\hat{\theta}_n) \enspace.
\end{align}

To adapt to an unknown changepoint $\tau$, a natural statistic is $R_\text{lin} := \max_{\tau \in [n-1]} R_n(\tau)$.
And, given a significance level $\alpha$, the decision rule reads $\psi_\text{lin}(\alpha) := \ind\{R_\text{lin} > H_\text{lin}(\alpha)\}$,
where $H_\text{lin}(\alpha)$ is a prescribed threshold discussed in \cref{sec:level_and_power}.
We call $R_\lin$ the \emph{linear statistic} and $\psi_\text{lin}$ the \emph{linear test}.

\subsection{Sparse alternatives}
There are cases when the jump only happens in a small subset of components of $\tpar$.
The linear test, which is built assuming the jump is large, may fail to detect such small jumps.
Therefore, we also consider \emph{sparse alternatives}.
\vspace{0.2cm}
\begin{enumerate}[resume,label=(P\arabic*),itemsep=1pt,parsep=0pt,topsep=0pt,partopsep=0pt]
\item \label{prob:P1} Testing the presence of a small jump:
\end{enumerate}
\begin{align*}
  \hnull: &\ \theta_k = \theta_0 \mbox{ for all } k = 1,\dots,n \\
  \halt:&\ \mbox{after some time $\tau$, $\theta_k$ jumps from $\theta_0$ to $\theta_0 + \Delta$,} \\
  &\ \mbox{where $\Delta$ has at most $P$ nonzero entries} \enspace.
\end{align*}
Here $P$ is referred to as the \emph{maximum cardinality}, which is set to be much smaller than $d$, the dimension of $\theta$.
We denote by $T$ the changed components, \ie $\Delta_{T} \neq 0$ and $\Delta_{[d]\backslash T} = 0$.

Given a fixed $T$, we consider the \emph{truncated statistic}
\[
  R_n(\tau, T) = S_{\post{\tau+1}}^\top(\htheta_n)_T
  \big[\info_n(\htheta_n; \tau)_{T,T}\big]^{-1}
  S_{\post{\tau+1}}(\htheta_n)_T \enspace.
\]
Let $\mathcal{T}_p$ be the collection of all subsets of size $p$ of $[d]$.
To adapt to unknown $T$, we use
\begin{align}\label{eq:unknown_T_stat}
  R_n(\tau, P; \alpha) := \max_{p \in [P]} \max_{T \in \mathcal{T}_p} H_p(\alpha)^{-1}R_n(\tau, T) \enspace,
\end{align}
where we use a different threshold $H_p(\alpha)$ for each $p \in [P]$.
Finally, since $\tau$ is also unknown, we propose $R_\text{scan}(\alpha) := \max_{\tau \in [n-1]} R_n(\tau, P; \alpha)$,
with decision rule $\psi_\text{scan}(\alpha) := \ind\{R_\text{scan}(\alpha) > 1\}$.
We call $R_\text{scan}(\alpha)$ the \emph{scan statistic} and $\psi_\text{scan}$ the \emph{scan test}.

To combine the respective strengths of these two tests, we consider the test
\begin{align}\label{eq:autograd_stat}
  \psi_{\text{auto}}(\alpha) : = \max\{\psi_\text{lin}(\alpha_l), \psi_\text{scan}(\alpha_s)\} \enspace,
\end{align}
with $\alpha_l + \alpha_s = \alpha$, and we refer to it as the \emph{auto-test}.
The choice of $\alpha_l$ and $\alpha_s$ should be based on prior knowledge regarding how likely the jump is small.
We illustrate how to monitor a learning machine with \emph{auto-test} in \cref{fig:monitoring}.

\begin{algorithm}[t]
  \centering
  \caption{Auto-test}
  \label{alg:autograd}
  \begin{algorithmic}[1]
    \STATE {\bfseries Input:} data $(\obs_i)_{i=1}^n$, log-likelihood $\ell$, levels $\alpha_l$ and $\alpha_s$, and maximum cardinality $P$.
    \FOR{$\tau = 1$ {\bfseries to} $n-1$}
    \STATE Compute $R_n(\tau)$ in \eqref{eq:score_stat} using AutoDiff.
    \STATE Compute $R_n(\tau, P; \alpha_s)$ in \eqref{eq:unknown_T_stat}.
    \ENDFOR
    \STATE {\bfseries Output:} $\psi_{\text{auto}}(\alpha) = \max\{\psi_\text{lin}(\alpha_l), \psi_\text{scan}(\alpha_s)\}$ in \eqref{eq:autograd_stat}.
  \end{algorithmic}
\end{algorithm}

\subsection{Differentiable programming}
An attractive feature of~\emph{auto-test} is that it can be computed by inverse-Hessian-vector products.
That opens up the possibility to implement it easily using a machine learning library designed within a differentiable programming framework. 
Indeed, the inverse-Hessian-vector product can then be efficiently computed via automatic differentiation; see \Cref{sec:implementation} for more details.
The algorithm to compute the \emph{auto-test} is presented in \cref{alg:autograd}.

%% file: sections/sec3.tex
We summarize the asymptotic behavior of the proposed score-based statistics under null and alternatives.
The precise statements and proofs can be found in \Cref{sec:proofs}.

\begin{proposition*}[Null hypothesis]\label{thm:null}
  Under the null hypothesis and certain conditions, we have, for any subset $T \subset [d]$ and $\tau_n \in \mathbb{N}$ such that $\tau_n / n \rightarrow \lambda \in (0, 1)$,
  \[
    R_n(\tau_n) \rightarrow_d \chi_d^2 \quad \mbox{and} \quad R_n(\tau_n, T) \rightarrow_d \chi_{|T|}^2 \enspace,
  \]
  where we denote by $\rightarrow_d$ the convergence in distribution.
  In particular, with thresholds $H_\text{lin}(\alpha) = q_{\chi_d^2}(\alpha/n)$ and $H_p(\alpha) = q_{\chi_p^2}\big(\alpha/[\binom{d}{p}n(p+1)^2]\big)$, the tests $\psi_{\text{lin}}(\alpha)$, $\psi_{\text{scan}}(\alpha)$ and $\psi_{\text{auto}}(\alpha)$ are consistent in level, where $q_{D}(\alpha)$ is the upper $\alpha$-quantile of the distribution $D$.
\end{proposition*}

Most conditions in the above Proposition are standard.
In fact, under suitable regularity conditions, they hold true for \emph{i.i.d.}~models, hidden Markov models \cite[Chapter 12]{bickel1998asymptotic}, and stationary autoregressive moving-average models \cite[Chapter 13]{douc2014nonlinear}.

The next proposition verifies the consistency in power of the proposed tests under fixed alternatives.
\begin{proposition*}[Fixed alternative hypothesis]
\label{thm:fix_alternative}
  Assume the observations are independent, and the alternative hypothesis is true with a fixed change parameter $\Delta$.
  Let the changepoint $\tau_n$ be such that $\tau_n/n \rightarrow \lambda \in (0, 1)$.
  Under certain conditions, the tests $\psi_{\text{lin}}(\alpha)$, $\psi_{\text{scan}}(\alpha)$ and $\psi_{\text{auto}}(\alpha)$ are consistent in power.
\end{proposition*}

%% file: sections/sec4.tex
\begin{figure}[t]
  \centering
  \begin{minipage}{0.47\textwidth}
    \centering
    \includegraphics[width=0.49\linewidth]{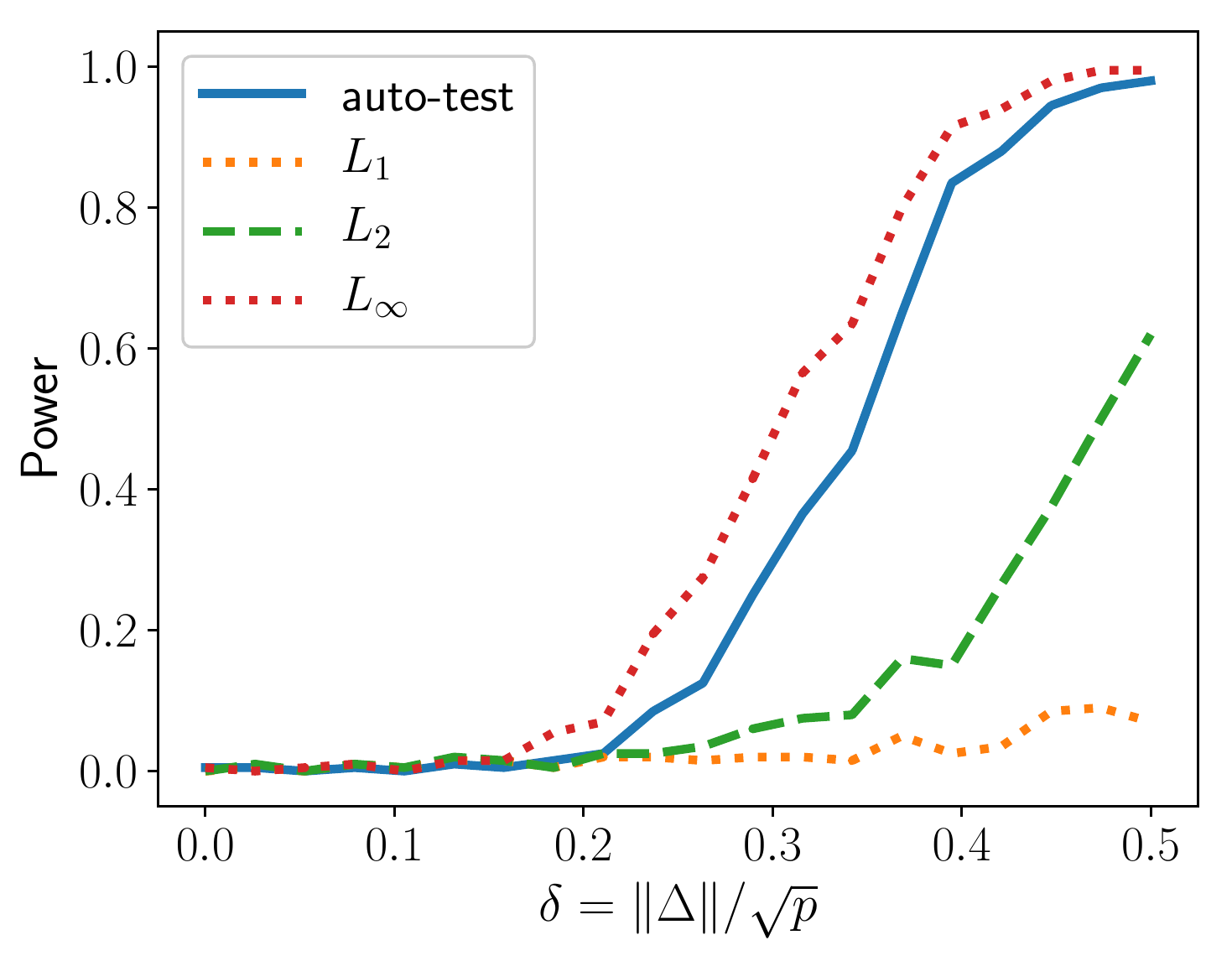}
    \includegraphics[width=0.49\linewidth]{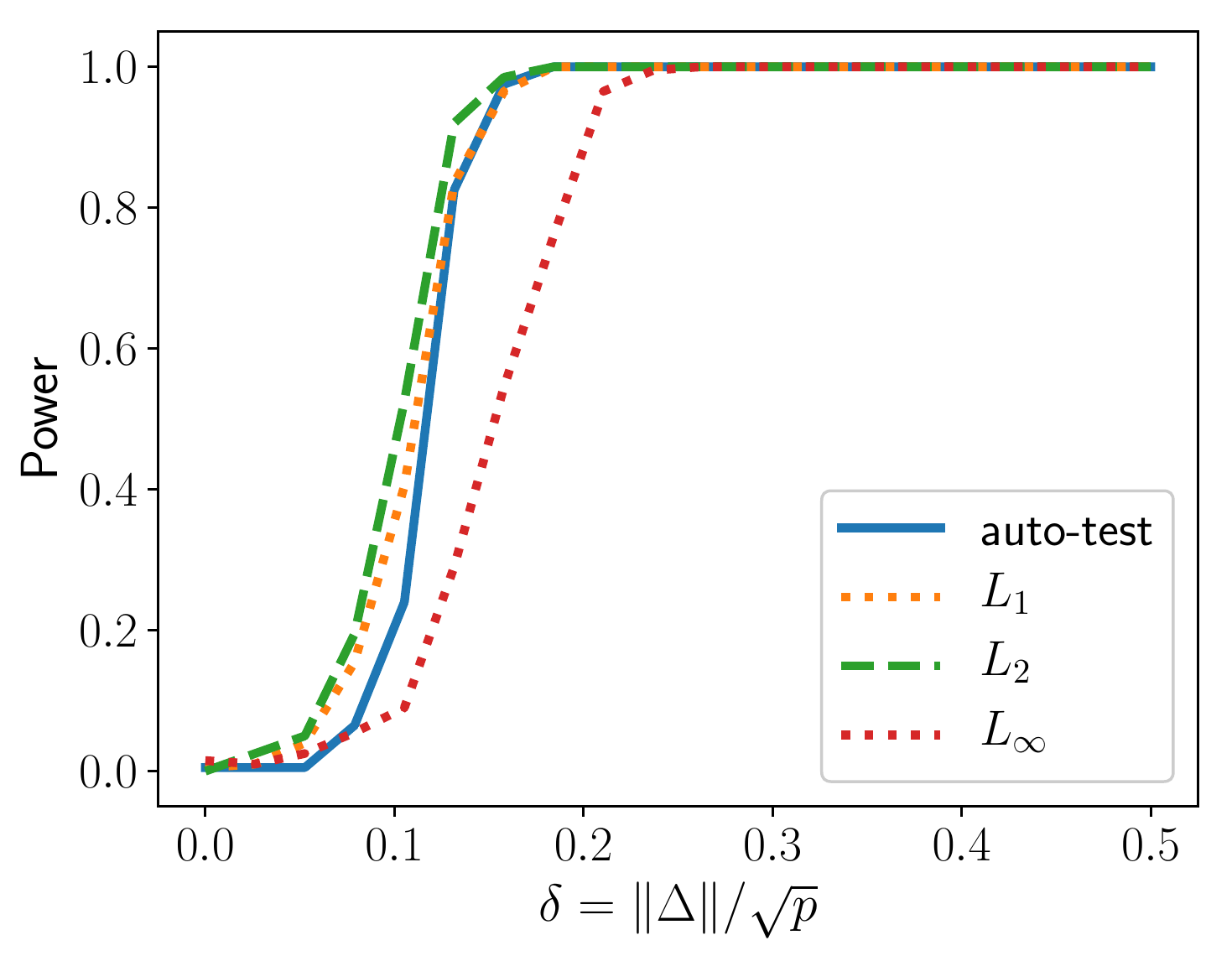}
    \caption{Power curves for a linear model with $d = 101$ (left: $p=1$; right: $p=20$). The sample size is $n = 1000$.}
    \label{fig:linear}
  \end{minipage}
  \hspace{0.5cm}
  \begin{minipage}{0.47\textwidth}
    \centering
    \includegraphics[width=0.49\linewidth]{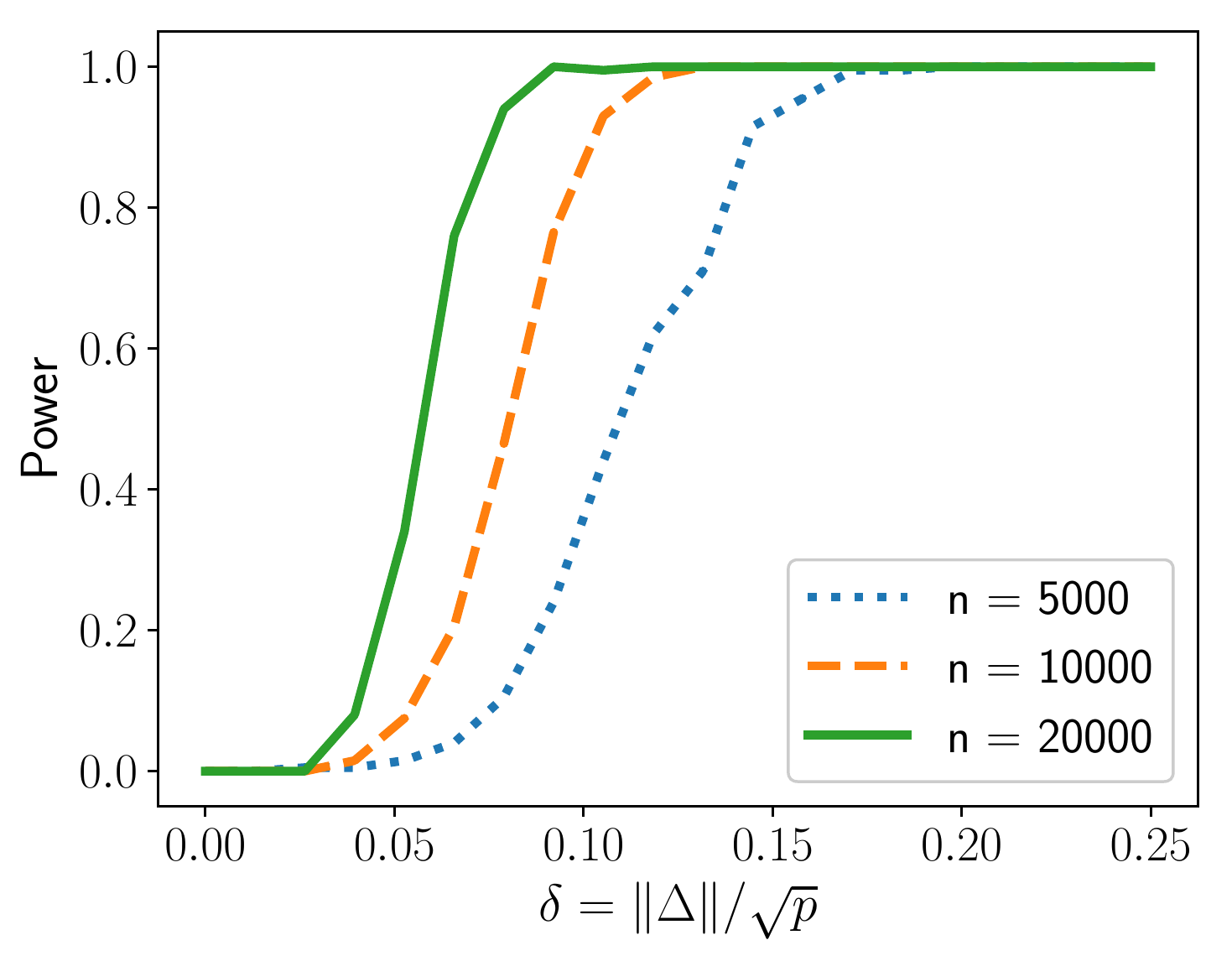}
    \includegraphics[width=0.49\linewidth]{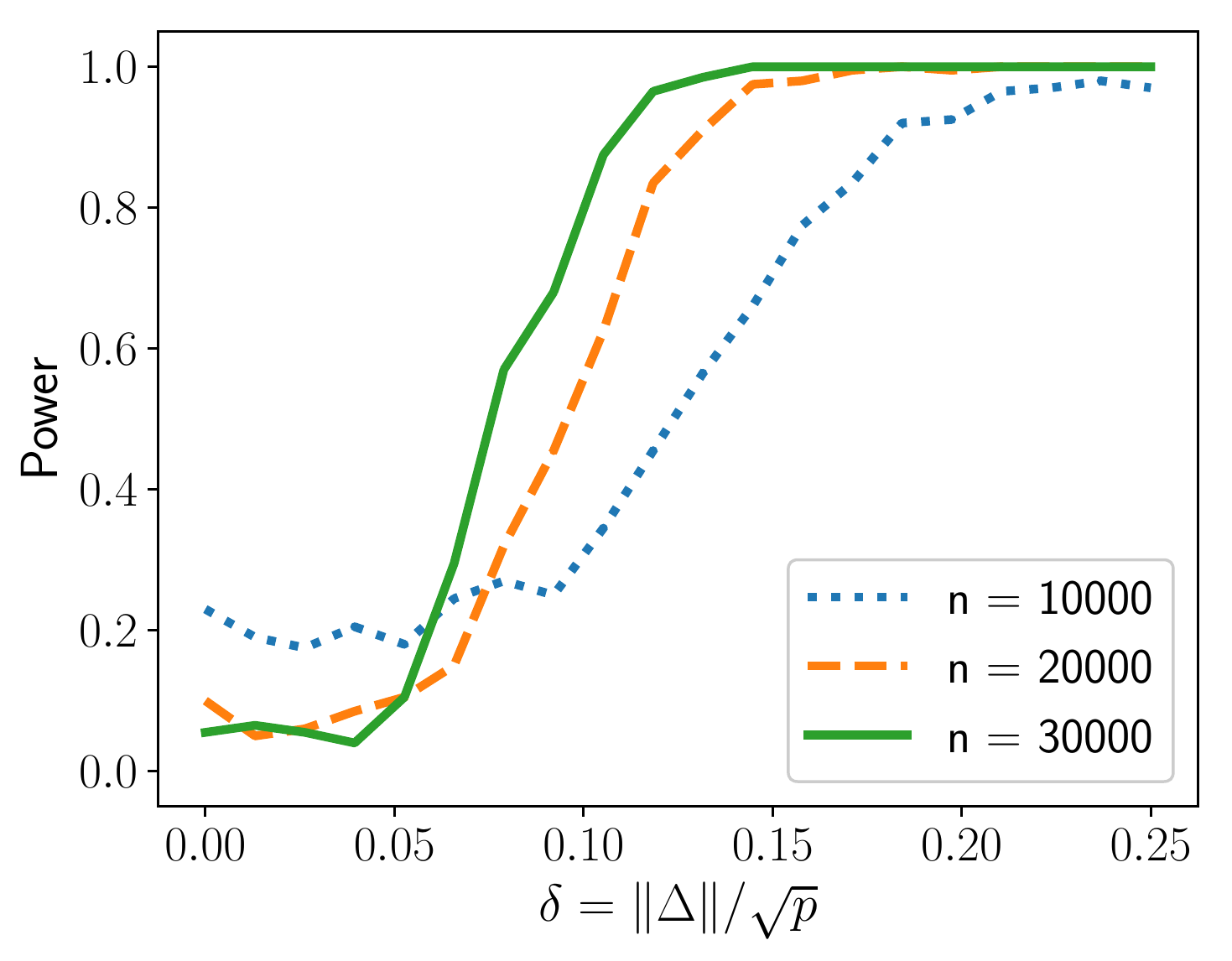}
    \caption{Power curves of the \emph{auto-test} for a text topic model with $p=1$ (left: $(N, M) = (3, 6)$; right: $(N, M) = (7, 20)$).}
    \label{fig:dependent}
  \end{minipage}
\end{figure}

We apply our approach to detect changes on synthetic data and on real data.
We summarize the settings and our findings.
More details and additional results are deferred to \Cref{sec:experiment}.

\textbf{Synthetic data.}
For each model, we generate the first half sample from the pre-change parameter $\theta_0$ and generate the second half from the post-change parameter $\theta_1$, where $\theta_1$ is obtained by adding $\delta$ to the first $p$ components of $\theta_0$.
Next, we run the proposed \emph{auto-test} to monitor the learning process,
where the significance levels are set to be $\alpha = 2\alpha_l = 2\alpha_s = 0.05$ and the maximum cardinality $P = \lfloor \sqrt{d} \rfloor$.
We repeat this procedure 200 times and approximate the detection power by rejection frequency.
Finally, we plot the power curves by varying $\delta$.
Note that the value at $\delta = 0$ is the empirical false alarm rate.

\textbf{Additive model.}
We consider a linear model with $101$ parameters and investigate two sparsity levels, $p=1$ and $p=20$.
We compare the \emph{auto-test} with three baselines given by the $L_a$ norm of the score function for $a \in \{1, 2, \infty\}$, where these baselines are calibrated by the empirical quantiles of their limiting distributions.
Note that the linear test corresponds to the $L_2$ norm with a proper normalization.
And the scan test with $P = 1$ corresponds to the $L_\infty$ norm.
As shown in \cref{fig:linear},
when the change is sparse, \ie a small jump, the \emph{auto-test} and $L_\infty$ test have similar power curves and outperform the rest of the tests significantly.
When the change is less sparse, \ie a large jump, all tests' performance gets improved, with the $L_\infty$ test being less powerful than the other three.
This empirically illustrates that (1) the $L_\infty$ test work better in detecting sparse changes, (2) the $L_1$ test and the $L_2$ test are more powerful for non-sparse changes and
(3) the \emph{auto-test} achieves comparable performance in both situations.

The proposed \emph{auto-test} is calibrated by its large sample properties and the Bonferroni correction.
This strategy tends to result in tests that are too conservative, with empirical false alarm rates largely below 0.05.
We also use resampling-based strategy to calibrate the \emph{auto-test}, \ie generating bootstrap samples and calibrating the test using the quantiles of the test statistics evaluated on bootstrap samples.
The empirical false alarm rates are around 0.065 for both $p = 1$ and $p = 20$.

\begin{table}[t]
  \caption{Decision of the scan test on the TV-show application: each (row, column) pair stands for a concatenation; ``R'' means reject and ``N'' means not reject. Red entries are false alarms.}
  \label{tab:tv_rejection}
  \centering
  \begin{tabular}{lllllllll}
  \toprule
              & F1 & F2 & M1 & M2 & S1 & S2 & D1 & D2 \\
  \midrule
  F1 & N       & N       & N      & N      & R         & R         & R         & R         \\
  F2 & N       & N       & \textcolor{red}{\emph{R}}      & N      & R         & R         & R         & R         \\
  M1 & N       & \textcolor{red}{\emph{R}}       & N      & N      & R         & R         & R         & R         \\
  M2 & N       & N       & N      & N & R         & R         &R         & R         \\
  S1 & R        & R        & R       & R       & N       & N        &\textcolor{red}{\emph{R}}        & \textcolor{red}{\emph{R}}        \\
  S2 & R        & R        & R       & R       & N        & N        &\textcolor{red}{\emph{R}}        & \textcolor{red}{\emph{R}}        \\
  D1 & R        & R        & R       & R       & \textcolor{red}{\emph{R}}    & \textcolor{red}{\emph{R}}     &N        & \textcolor{red}{\emph{R}}        \\
  D2 & R        & R        & R       & R       & \textcolor{red}{\emph{R}}    & \textcolor{red}{\emph{R}}     &N        & N        \\
  \bottomrule
  \end{tabular}
  \vskip -0.1in
\end{table}

\textbf{Text topic model.}
We consider a text topic model \cite{stratos2015model} and investigate the \emph{auto-test} for different sample sizes.
This model is a hidden Markov model whose emission distribution has a special structure.
We examine two parameter schemes: $(N, M) \in \{(3, 6), (7, 20)\}$, where $N$ is the number of hidden states and $M$ is the number of categories of the emission distribution,
and $p$ is set to be $1$.
As demonstrated in \cref{fig:dependent},
for the first scheme, all tests have small false alarm rates, and their power rises as the sample size increases.
For the second scheme, the false alarm rate is out of control in the beginning, but this problem is alleviated as the sample size increases.
This empirically verifies that the \emph{auto-test} is consistent in both level and power even for dependent data.

\textbf{Real data.}
We collect subtitles of the first two seasons of four TV shows---Friends (F), Modern Family (M), the Sopranos (S) and Deadwood (D)---where the former two are viewed as ``polite'' and the latter two as ``rude''.
For every pair of seasons, we concatenate them, and train the text topic model with $N = \lfloor \sqrt{n/100} \rfloor$ and $M$ being the size of vocabulary built from the training corpus.
The task is to detect changes in the rudeness level.
As an analogy, the text topic model here corresponds to a chatbot, and subtitles are viewed as interactions with users.
We want to know whether the conversation gets rude as the chatbot learns from the data.

The linear test, \ie the \emph{auto-test} with $\alpha_l = \alpha$ and $\alpha_s = 0$, does a perfect job in reporting shifts in rudeness level.
However, it has a high false alarm rate ($27 / 32$).
This is expected since the linear test may capture the difference in other aspects, \eg topics of the conversation.
The scan test, \ie the \emph{auto-test} with $\alpha_l = 0$ and $\alpha_s = \alpha$, has much lower false alarm rate ($11/32$).
Moreover, as shown in \cref{tab:tv_rejection}, there are only two false alarms in the most interesting case, where the sequence starts with a polite show.
We note that this problem is hard, since rudeness is not the only factor that contributes to the difference between two shows.
The results are promising since we benefit from exploiting the sparsity even without knowing which model components are related to the rudeness level.

%% file: sections/implement.tex
For simplicity of the notation, we write $\hat S_{i:j} = S_{i:j}(\est_n)$ and $\hat \calI_{i:j} = \calI_{i:j}(\est_n)$ throughout this section.

\subsection{Algorithmic aspects}
\label{sub:implement}

\begin{algorithm}[b]
  \centering
  \caption{Linear statistic with the direct approach}
  \label{alg:linear_naive}
  \begin{algorithmic}[1]
    \STATE {\bfseries Input:} Data $(W_k)_{k=1}^n$, log-likelihood $\ell$, and MLE $\est_n$.
    \STATE Compute $\hat S_{1:n}$ by calling AutoDiff on $\ell_{1:n}(\est_n)$.
    \STATE Compute $\hat \calI_{1:n}$ by calling $d$ times AutoDiff on $\hat S_{1:n}$.
    \STATE Compute $\hat \calI_{1:n}^{-1}$.
    \FOR{$\tau = 1, \dots, n-1$}
      \STATE Compute $\hat S_{1:\tau}$ by calling AutoDiff on $\ell_{1:\tau}(\est_n)$, and then compute $\hat S_{\tau+1:n} = \hat S_{1:n} - \hat S_{1:\tau}$.
      \STATE Compute $\hat I_{1:\tau}$ by calling $d$ times AutoDiff on $\hat S_{1:\tau}$.
      \STATE Compute $R_n(\tau)$ in \eqref{eq:a:linear_stat}.
    \ENDFOR
    \STATE Compute $R_{\lin}$ in \eqref{eq:a:linear_stat}.
    \STATE {\bfseries Output:} $R_{\lin}$.
  \end{algorithmic}
\end{algorithm}

Recall that the computation of \emph{auto-test} boils down to the computation of the linear statistic
\begin{align}\label{eq:a:linear_stat}
  R_{\lin} := \max_{\tau \in [n-1]} R_n(\tau) := \max_{\tau \in [n-1]} \hat S_{\tau+1:n}^\top \hat{\calI}_{n, \tau}^{-1} \hat S_{\tau+1:n} \enspace,
\end{align}
where $\hat \calI_{n,\tau} = \hat \calI_{1:\tau} - \hat \calI_{1:\tau} \hat \calI_{1:n}^{-1} \hat \calI_{1:\tau}$,
and the scan statistic
\begin{align}\label{eq:a:scan_stat}
  R_{\text{scan}} := \max_{\tau \in [n-1]} \max_{T \subset [d], |T| \le P} R_n(\tau, T) := \max_{\tau \in [n-1]} \max_{T \subset [d], |T| \le P} [\hat{S}_{\tau+1:n}]_{T}^\top [\hat{\calI}_{n, \tau}]_{T, T}^{-1} [\hat{S}_{\tau+1:n}]_{T} \enspace,
\end{align}
where $[\hat{\calI}_{n, \tau}]_{T, T}^{-1}$ should be understood as $\{[\hat{\calI}_{n, \tau}]_{T, T}\}^{-1}$.

To compute these two statistics, a direct approach is to compute the full Fisher information matrices and then invert them.
Another approach consists in solving the linear system $\calI^{-1} S$ by the conjugate gradient algorithm.
We refer to it as the AutoDiff-friendly approach.

In the following, we analyze the time and space complexity of these two approaches in the most general case, that is, the sequence $\{\hat \calI_{1:t}\}_{t=1}^n$ does not admit a recursion that could simplify its computation.
For every $t \in [n]$, we assume the computational graph of the log-likelihood is of size $tC_1$ with $C_1 \ge d$.
As a result, computing $\hat{S}_{1:t}$ by AutoDiff takes $\bigO{tC_1}$ time and $\bigO{tC_1}$ space.
Similarly, we assume the computational graph of the score $\hat S_{1:t}$ is of size $t C_2$.
Then the time and the space complexity of computing $\hat{\calI}_{1:t}(\theta)$ are $\bigO{tdC_2}$ and $\bigO{t C_2}$, respectively, if we call AutoDiff on $\hat S_{1:t}^\top e_k$ for each $k \in [d]$, where $\{e_k\}_{k=1}^d$ is the standard basis of $\bbR^d$.
We usually have $C_2 > C_1$ when $\ell(\theta)$ is not linear in $\theta$.

\textbf{Computing the linear statistic using automatic differentiation.}
The main steps to compute the linear statistic with the direct approach are summarized in \Cref{alg:linear_naive}.
The most time-consuming step is the for loop in steps 5-9.
For each $\tau \in [n-1]$, steps 6-8 take time $\bigO{\tau C_1}$, $\bigO{\tau d C_2}$ and $\bigO{d^3}$, respectively.
Therefore, the overall time complexity of \Cref{alg:linear_naive} is $\bigO{n^2 d C_2 + nd^3}$.
The most space-consuming steps are to store the computational graph of $\hat S_{1:n}$ with complexity $\bigO{n C_2}$, and to store the full Fisher information matrix with complexity $\bigO{d^2}$.
Consequently, the overall space complexity is $\bigO{nC_2 + d^2}$.

We now investigate the AutoDiff-friendly approach.
According to the Woodbury matrix identity, we have
\begin{align*}
  \hat \calI_{n,\tau}^{-1} = \hat \calI_{1:\tau}^{-1} + \calI_{\tau+1:n}^{-1} \enspace.
\end{align*}
The statistic $R_n(\tau)$ then reads
\begin{align}\label{eq:a:linear_symmetric}
  R_n(\tau) = \hat S_{\tau+1:n}^\top \hat \calI_{1:\tau}^{-1} \hat S_{\tau+1:n} + \hat S_{\tau+1:n}^\top \hat \calI_{\tau+1:n}^{-1} \hat S_{\tau+1:n} \enspace.
\end{align}
To compute $\hat \calI_{1:\tau}^{-1} \hat S_{\tau+1:n}$, we apply the conjugate gradient algorithm to solve the problem
\begin{align*}
  \min_{x} \left\{\frac12 x^\top \hat \calI_{1:\tau} x - \hat S_{\tau+1:n}^\top x \right\} \enspace.
\end{align*}
Each iteration of the conjugate gradient algorithm requires evaluating $\hat \calI_{1:\tau} x$, which can be obtained by calling AutoDiff on $\hat S_{1:\tau}^\top x$ with $\bigO{\tau C_2}$ time and $\bigO{\tau C_2}$ space.
Moreover, it converges in $M \le d$ steps.
As a result, computing $\hat \calI_{1:\tau}^{-1} \hat S_{\tau+1:n}$ takes $\bigO{\tau M C_2}$ time and $\bigO{\tau C_2}$ space.
The steps to compute $\hat \calI_{\tau+1:n}^{-1} \hat S_{\tau+1:n}$ is similar since $\hat \calI_{\tau+1:n} x = \hat \calI_{1:n} x - \hat \calI_{1:\tau} x$.
Hence, we may compute $R_\text{lin}$ as in \Cref{alg:linear_ours}.
The most expensive steps are the computation of $\hat \calI_{1:n}$ in step 3 and the for loop in steps 4-7.
Step 3 takes $\bigO{n d C_2}$ time and $\bigO{nC_2 + d^2}$ space.
For each $\tau \in [n-1]$, the steps within the for loop, as discussed above, take $\bigO{\tau M C_2}$ time and $\bigO{\tau C_2}$ space.
Hence, the overall time and space complexities are $\bigO{n^2 M C_2 + ndC_2}$ and $\bigO{n C_2 + d^2}$.
Since $M \le d$, it is clear that this approach is more efficient than the direct one.

\begin{algorithm}[t]
  \centering
  \caption{Linear statistic with the conjugate gradient algorithm}
  \label{alg:linear_ours}
  \begin{algorithmic}[1]
    \STATE {\bfseries Input:} Data $(W_k)_{k=1}^n$, log-likelihood $\ell$, and MLE $\est_n$.
    \STATE Compute $\hat S_{1:n}$ by calling AutoDiff on $\ell_{1:n}(\est_n)$.
    \STATE Compute $\hat \calI_{1:n}$ by calling $d$ times AutoDiff on $\hat S_{1:n}$.
    \FOR{$\tau = 1, \dots, n-1$}
      \STATE Compute $\hat S_{1:\tau}$ by calling AutoDiff on $\ell_{1:\tau}(\est_n)$, and then compute $\hat S_{\tau+1:n} = \hat S_{1:n} - \hat S_{1:\tau}$.
      \STATE Compute $R_n(\tau)$ in \eqref{eq:a:linear_symmetric} by the conjugate gradient algorithm.
    \ENDFOR
    \STATE Compute $R_{\lin}$ in \eqref{eq:a:linear_stat}.
    \STATE {\bfseries Output:} $R_{\lin}$.
  \end{algorithmic}
\end{algorithm}

\textbf{Computing the scan statistic using automatic differentiation.}
Computing the scan statistic exactly may be exponentially expensive in the parameter dimension $d$, since it involves a maximization over all subsets of $[d]$ with cardinality $p \le P$.
Alternatively, we approximate the maximizer of $\max_{\abs{T}=p} R_n(\tau, T)$, say $T_p$, by the indices of the largest $p$ components in
\begin{align}\label{eq:diag_approx}
  v(\tau) := \hat S_{\post{\tau+1}}^\top \text{diag}\{\hat \calI_{n, \tau}\}^{-1} \hat S_{\post{\tau+1}} \enspace.
\end{align}
That is, we consider all $T$ with $\abs{T} = 1$, and approximate the maximizer $T_p$ by the union of the ones that give the largest $p$ values of $R_n(\tau, T)$.
We show in \Cref{sec:proofs} that this approximation is accurate if the difference between the largest eigenvalue and the smallest eigenvalue of $\hat \calI_{n, \tau}$ is small compared to $\lVert \hat S_{\tau+1:n} \rVert^2$.
Formally, we approximate $R_{\scan}$ by
\begin{align}\label{eq:a:approx_scan}
  R_{\scan} \approx \max_{\tau \in [n-1]} \max_{p \le P} R_n(\tau, T_{\tau, p}) := \max_{\tau \in [n-1]} \max_{p \le P}\ [\hat S_{\tau+1:n}]_{T_{\tau, p}}^\top [\hat \calI_{n,\tau}]_{T_{\tau, p}, T_{\tau, p}}^{-1} [\hat S_{\tau+1:n}]_{T_{\tau, p}},
\end{align}
where $T_{\tau, p}$ corresponds to the largest $p$ indices of $v(\tau)$.

Note that, in order to compute the scan statistic in a similar fashion as the linear statistic, we may modify the normalizing matrix $[\hat \calI_{n, \tau}]_{T, T}^{-1}$ as
\begin{align}\label{eq:a:modify_normalize}
  [\hat \calI_{1:\tau}]_{T, T}^{-1} + [\hat \calI_{\tau+1:n}]_{T, T}^{-1} \enspace.
\end{align}
It can be shown that \eqref{eq:a:modify_normalize} converges to the same limit as $\hat \calI_{1:\tau}$ under the null so that the calibration discussed in \Cref{sec:proofs} remains valid.
Hence, for the direct approach, we need the following steps to compute $R_{\scan}$ in addition to \Cref{alg:linear_naive}: 1) sort $v(\tau)$ and obtain $\{T_{\tau, p}\}_{p \in [P]}$, and 2) compute $\{R_n(\tau, T_{\tau, p})\}_{p \in [P]}$, for each $\tau \in [n-1]$.
For the AutoDiff-friendly approach, after we sort $v(\tau)$, we can again compute $[\hat \calI_{1:\tau}]_{T_{\tau, p}, T_{\tau, p}}^{-1} [\hat S_{\tau+1:n}]_{T_{\tau, p}}$ by the conjugate gradient algorithm.
Since $P \ll d$, the time complexity and space complexity of the linear statistic dominate the ones of the scan statistic.

When the observations $\{W_k\}_{k=1}^n$ are independent, the score $\hat S_{1:\tau}$ and information $\hat \calI_{1:\tau}$ can be computed recursively in the direct approach.
As a result, computing the \emph{auto-test} with the direct approach will be more efficient if $n \gg d$.

\subsection{Running times}
\label{sub:running_time}

We then compare empirically the running time of the two approaches.
For simplicity, we focus on applying them to compute $\hat \calI_{1:n}^{-1} z$ for some randomly generated vector $z \in \bbR^d$.
We consider two models: 1) a linear model $Y = \theta^\top X + \varepsilon$ with log-likelihood (up to a constant) $\ell(\theta) = -(Y - \theta^\top X)^2$ (or quadratic loss), where $\theta \in \mathbb{R}^d$;
2) a multilayer perceptron (MLP) with the following structure: $x_0 \rightarrow x_1 = \sigma(A_1x_0 + b_1) \rightarrow x_2 = A_2 x_1 + b_2$, where $x_0 \in \bbR^r$ is the input vector, $A_1 \in \bbR^{r \times [r/2]}$, $b_1 \in \bbR^{[r/2]}$, $A_2 \in \bbR^{[r/2] \times 1}$ and $b_2 \in \bbR$.
Hence, there are $d = r[r/2] + 2[r/2] + 1$ parameters in this model.
The loss function is again chosen as the quadratic loss.
For each of the two models,
we generate $n$ \iid observations from this model and use the two approaches (``Direct'' and ``Ours'') to compute $\hat \calI_{1:n}^{-1}z$.
For the conjugate gradient algorithm, we set the target accuracy to be $10^{-7}$ and set the maximum number of iterations to be $2d$.
For each pair of $(n, d)$, we repeat the experiment $5$ times and report the average running time with standard error in \cref{fig:runtime}.
The experiments are performed on a machine with 32 2.8GHz Intel Core i9 CPUs.

For the linear model, the information matrix is well-conditioned, so it took strictly less than $d$ iterations for the conjugate gradient algorithm to converge.
This contributes to the significant improvement on the running time compared to the direct approach.
As for the MLP, the information matrix is ill-conditioned, so it usually took the conjugate gradient algorithm the maximum number of iterations, \ie $2d$, to converge.
In fact, the running time of the AutoDiff-friendly approach is about twice larger than the direct approach.
Note that this time could be potentially reduced by computing the inverse-matrix-vector product inexactly.

\begin{figure}[t]
  \centering
  \includegraphics[width=\textwidth]{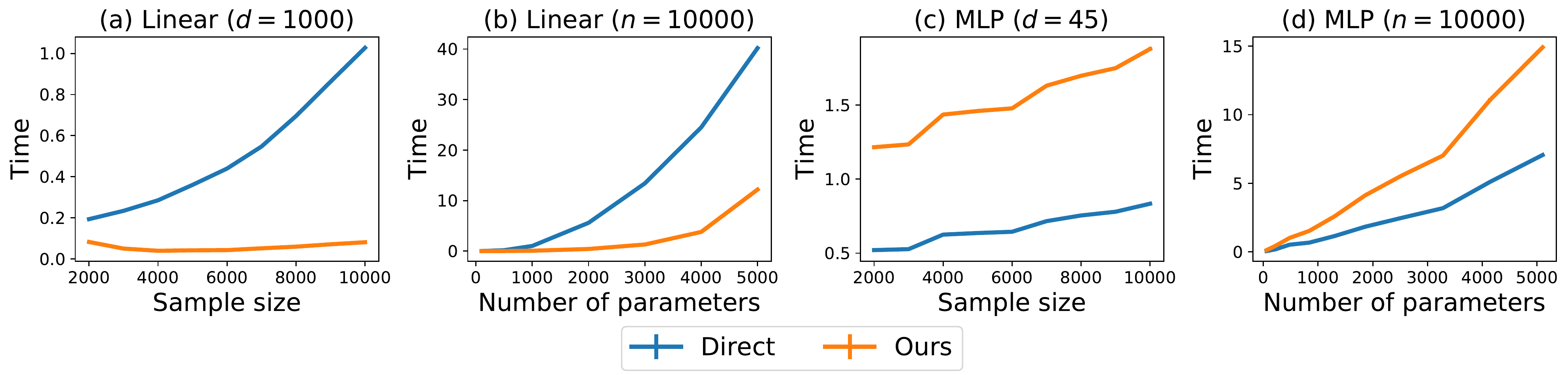}
  \caption{\textbf{(a)} Running time versus sample size for linear models with $d = 1000$; \textbf{(b)} Running time versus number of parameters for linear models with $n=10000$; \textbf{(c)} Running time versus sample size for multilayer perceptrons with $d = 1035$ ($r=45$); \textbf{(d)} Running time versus number of parameters for multilayer perceptrons with $n=10000$.}
  \label{fig:runtime}
\end{figure}

\subsection{Examples}
\label{sub:examples}

\begin{example}[Text topic model]
  The text topic model introduced in \cite{stratos2015model} is a hidden Markov model with transition probability $q$ and emission probability $g$, supported respectively on finite sets $[N]$ and $[M]$.
  Moreover, it satisfies the so-called Brown assumption: for each observation $X \in [M]$, there exists a unique hidden state $\calH(X) \in [N]$ such that $g(X|\calH(X)) > 0$ and $g(X|h) = 0$ for all $h \neq \calH(X)$.
  The authors proposed a class of spectral methods to recover approximately the map $\hat{\calH}$ up to permutation.
  Consequently, the log-likelihood can be computed as
  \[
    \ell_n(\theta) = \sum_{k=1}^n \log{q(\hat{\calH}_k|\hat{\calH}_{k-1})} + \log{g(X_k|\hat{\calH}_k)} \enspace,
  \]
  where $X_0$ is assumed to be known.
\end{example}

\begin{example}[Time series model]
  Consider an autoregressive moving-average (ARMA) model:
  \[
    X_t = \sum_{i=1}^r \phi_i X_{t-i} + \varepsilon_t + \sum_{i=1}^q \varphi_i \varepsilon_{t-i} \enspace,
  \]
  where $\{\varepsilon_t\}$ are i.i.d.~standard norm random variables.
  Let $\theta = (\phi; \varphi)$.
  Assume that $r \ge q$ and $X_{1:r}$ is completely known.
  Then the log-likelihood reads:
  \[
    \ell_n(\theta) = -\frac1{2} \sum_{t=r+1}^n \varepsilon_t^2 + C \enspace,
  \]
  where $\varepsilon_t = X_t - \sum_{i=1}^r \phi_i X_{t-i} - \sum_{i=1}^q \varphi_i \varepsilon_{t-i}$.
\end{example}

\begin{example}[Hidden Markov model]
  Suppose that observations $(Y_k)_{k=1}^n$ are from a hidden Markov model (HMM)
  \begin{align*}
    X_k \sim Q(X_{k-1}, \cdot) \quad \mbox{and} \quad Y_k \sim G(X_k, \cdot),
  \end{align*}
  where $G$ and $Q$ are the transition distribution and emission distribution, respectively.
  For simplicity, we write $q_{x, x'} = Q(x, x')$ and $g_k(x) = G(x, Y_k)$.
  Its log-likelihood function can be computed by the \emph{normalized forward recursion} \cite[Chapter 3]{Cappe2005inference}: $\ell_n(\theta) = \sum_{k=1}^n \log{c_k}$ where, recursively,
  \begin{align*}
      c_k &= \sum_{x_{k-1},x_k=1}^M\phi_{k-1}(x_{k-1}) q_{x_{k-1}, x_k} g_k(x_k) \\
      \phi_k(x_k) &= c_k^{-1} \sum_{x_{k-1} = 1}^M \phi_{k-1}(x_{k-1}) q_{x_{k-1}, x_k} g_k(x_k), \quad \forall x_k \in [M] \enspace,
  \end{align*}
  with initial conditions
  \begin{align*}
      c_0 &= \sum_{x_0=1}^M g_0(x_0)\nu(x_0) \\
      \phi_0(x_0) &= c_0^{-1} g_0(x_0)\nu(x_0), \quad \forall x_0 \in [M] \enspace.
  \end{align*}
\end{example}

%% file: sections/proofs.tex
\subsection{Null hypothesis}

This section is devoted to determine thresholds for the linear test, scan test, and \emph{auto-test} so that they are consistent in level.
For this purpose, we first derive the limiting distribution of $R_n(\tau_n)$ for any sequence $(\tau_n)_{n \ge 1}$ such that $\tau_n/n \rightarrow \lambda \in (0, 1)$.
We then determine the thresholds based on the limiting distribution.

\begin{assumption}\label{asmp:null}
  Let $\obs_1,\dots,\obs_n$ be a time series with a correctly specified model $\{p_{\theta}: \theta \in \Theta \subset \bbR^d\}$.
  Suppose that the true parameter $\theta_0 \in \interior(\Theta)$ (the interior of $\Theta$), and that the following assumptions hold:
  \begin{enumerate}[label=A\arabic*,itemsep=1pt,parsep=0pt,topsep=0pt,partopsep=0pt]
    \item\label{cond:1}: ~
    $\Theta$ contains an open neighborhood $\Theta_0$ on which $\ell_n(\theta) := \log p_{\theta}(\obs_1,\dots,\obs_n)$ is twice continuously differentiable and $\lVert n^{-1} \nabla^3_\theta \ell_n(\theta) \rVert \overset{a.s.}{\le} M(W_1, \dots, W_n) = O_p(1)$ for every $\theta \in \Theta_0$.
    \item \label{cond:2}: ~
    $- \nabla_{\theta}^2 \ell_n(\theta_0)/n \rightarrow_p \mathcal{I}_0$ where $\mathcal{I}_0 \in \bbR^{d\times d}$ is positive definite.
    \item \label{cond:3}: ~
    The MLE $\htheta_n$ exists and $\sqrt{n}(\htheta_n - \theta_0) \rightarrow_d \mathcal{N}(0, \mathcal{I}_0^{-1})$ (convergence in distribution).
    \item \label{cond:4}: ~
    The normalized score can be written as a sum of a martingale difference sequence, up to an $o_p(1)$ term, \emph{\wrt}to some filtration $\{\mathcal{F}_t\}_{t \in \Z}$, that is,
    \begin{align*}
      Z_n(\theta_0)
      := \frac{1}{\sqrt{n}}S_{n}(\theta_0)
        = \frac{1}{\sqrt{n}}\nabla_{\theta} \ell_{n}(\theta_0)
        = \sum_{k=1}^n \frac{M_k}{\sqrt{n}} + o_p(1) \enspace,
    \end{align*}
    where $\Ex[M_k|\mathcal{F}_{k-1}] = 0, \forall k \in [n]$.
    In addition, this martingale difference sequence satisfies the Lindeberg conditions:
    \begin{enumerate}[label=\ref*{cond:4}-(\alph*)]
      \item \label{cond:4a}: ~ $n^{-1} \sum_{k=1}^n \Ex[M_k M_k^\top|\mathcal{F}_{k-1}] \rightarrow_p \mathcal{I}_0$ and
      \item \label{cond:4b}: ~ $\forall \varepsilon > 0$ and $\alpha \in \bbR^d,
        n^{-1} \sum_{k=1}^n
        \Ex\left[(\alpha^\top M_k)^2
        \ind{\{\lvert \alpha^\top M_k \rvert > \sqrt{n}\varepsilon\}} | \mathcal{F}_{k-1} \right] \rightarrow_p 0.$
    \end{enumerate}
  \end{enumerate}
\end{assumption}

A useful sufficient condition for Assumption \ref{cond:4} is given below.
\begin{lemma}
  Assume that the normalized score can be written as
  \begin{align*}
    Z_n(\theta_0) = \sum_{k=1}^n \frac{M_k}{\sqrt{n}} + o_p(1) \enspace,
  \end{align*}
  where $\{M_k\}_{k \in \mathbb{N}_+}$ is a stationary and ergodic martingale difference sequence \wrt its natural filtration, then Assumption \ref{cond:4} holds true. 
\end{lemma}
\begin{proof}
  By stationarity, there exists a fixed measurable function $f: \R^\infty \rightarrow \R^\infty$ such that, for all $k \in \bbN_+$,
    \begin{equation*}
      \mathbb{E}[M_k M_k^\top|M_{k-1},M_{k-2},\dots] = f(M_{k-1}, M_{k-2}, \dots)
    \end{equation*}
    almost surely.
    Due to the ergodicity of $M_k$, the series $N_k = f(M_{k-1}, M_{k-2}, \dots)$ is also ergodic so that $\bar{N}_n \rightarrow_{a.s.} \Ex[N_1]$, \ie the condition \ref{cond:4a} holds true.
    Similarly, given $c > 0$,
    \begin{equation*}
      G_n(c) := \frac1n\sum_{k=1}^n \Ex\left[ (\alpha^\top M_k)^2 | \ind\set{\abs{\alpha^\top M_k} > c} | \mathcal{F}_{k-1} \right] \rightarrow_{a.s.} G(c)
    \end{equation*}
    for any $\alpha \in \R^d$, where $G(c) = \bbE[(\alpha^\top M_1)^2 | \ind\set{\abs{\alpha^\top M_1} > c}]$ can be arbitrarily small by setting $c$ to be large.
    Hence, for any $\delta > 0$ and any $\alpha \in \R^d$, there exists a constant $c_0$ and an integer $N > 0$ such that $\forall n > N$, we have $G_n(c_0) < \delta$ almost surely.
    To verify the condition \ref{cond:4b}, note that $G_n(c)$ is decreasing in $c$, so,
    for every $\varepsilon > 0$, there exists $M > 0$ such that $n > M$ implies
    \begin{equation*}
      \frac1n \sum_{k=1}^n \Ex\left[ \alpha^\top M_k^2 | \ind\set{\abs{\alpha^\top M_k} > \varepsilon \sqrt{n}} | \mathcal{F}_{k-1} \right] \le G_n(c_0) < \delta
    \end{equation*}
    almost surely.
    As $\delta$ is arbitrary, we know that the condition \ref{cond:4b} holds.
\end{proof}

\begin{remark}
  Under suitable regularity conditions, \Cref{asmp:null} holds true for \emph{i.i.d.}~models, hidden Markov models \cite[Chapter 12]{bickel1998asymptotic}, and stationary autoregressive moving-average models \cite[Chapter 13]{douc2014nonlinear}.
\end{remark}

\begin{proposition}[Null hypothesis]\label{prop:null}
  Under \Cref{asmp:null}, we have, for any $\tau_n \in \bbN_+$ such that $\tau_n/n \rightarrow \lambda \in (0, 1)$,
  \begin{align*}
    \sqrt{\frac{n}{\tau_n(n - \tau_n)}} S_{\post{\tau_n+1}}(\htheta_n)
    \rightarrow_d
    \mathcal{N}(0, \mathcal{I}_0)\enspace.
  \end{align*}
  In particular,
  \begin{align*}
    R_n(\tau_n) &= S_{\post{\tau_n+1}}(\htheta_n)^\top \calI_n(\htheta_n; \tau_n)^{-1} S_{\post{\tau_n+1}}(\htheta_n) \rightarrow_d \chi_d^2 \\
    R_n(\tau_n, T) &= [S_{\tau_n+1:n}(\htheta_n)]_{T}^\top [\calI_{n}(\htheta_n; \tau_n)]_{T, T}^{-1} [S_{\tau_n+1:n}(\htheta_n)]_{T} \rightarrow_d \chi_{\abs{T}}^2, \quad \mbox{for any } T \subset [d] \enspace.
  \end{align*}
\end{proposition}

We start by showing that the partial observed information $\calI_n(\est_n; \tau_n)$ defined in \eqref{eq:schur_complement} is a consistent estimator of $\calI_0$ with proper normalization.

\begin{lemma}\label{lem:null_info}
  Under assumptions \ref{cond:1}-\ref{cond:3}, we have, for any $\tau_n \in \bbN_+$ such that $\tau_n/n \rightarrow \lambda \in (0, 1)$,
  \begin{align*}
    \frac{n}{\tau_n(n - \tau_n)} \calI_{n}(\htheta_n; \tau_n)
     \rightarrow_p \mathcal{I}_0 \enspace.
  \end{align*}
\end{lemma}
\begin{proof}
  According to Assumption \ref{cond:1} and Taylor's theorem, we obtain
  \begin{align*}
    \frac1n \norm{\nabla_\theta^2 \ell_n(\htheta_n) - \nabla_\theta^2 \ell_n(\theta_0)} \le \frac1n \norm{\nabla_\theta^3 \ell_n(\bar \theta_n)} \norm{\htheta_n - \theta_0} \enspace,
  \end{align*}
  where $\norm{\bar \theta_n - \theta_0} \le \norm{\htheta_n - \theta_0}$.
  Let $E_n$ be the event $\{\htheta_n \in \Theta_0\}$.
  By Assumption \ref{cond:3}, it holds that $\mathbb{P}(E_n) \rightarrow 0$, and thus
  \begin{align*}
    \frac1n \norm{\nabla_\theta^2 \ell_n(\htheta_n) - \nabla_\theta^2 \ell_n(\theta_0)} \le \norm{M(W_1, \dots, W_n)} \norm{\htheta_n - \theta_0} + o_p(1) = o_p(1) \enspace.
  \end{align*}
  Consequently, by the triangle inequality, we get
  \begin{align*}
    \norm{-\frac1n \nabla_\theta^2 \ell_n(\theta_n) - \calI_0} \le -\frac1n \norm{\nabla_\theta^2 \ell_n(\theta_n) - \nabla_\theta^2 \ell_n(\theta_0)} + \norm{-\frac1n \nabla_\theta^2 \ell_n(\theta_0) - \calI_0} \rightarrow_p 0 \enspace.
  \end{align*}
  This yields $-\nabla_\theta^2 \ell_n(\est_n)/n \rightarrow_p \calI_0$.
  It follows that
  \begin{align*}
    \frac1{n-\tau_n} \calI_{\tau_n+1:n}(\hat \theta_n)
    &= -\frac{1}{n - \tau_n} \nabla_{\theta}^2 \ell_{\post{\tau_n+1}}(\htheta_n)
    = -\frac{1}{n - \tau_n} [\nabla_{\theta}^2 \ell_{1:n}(\htheta_n) - \nabla_{\theta}^2 \ell_{1:\tau_n}(\htheta_n)] \\
    &\rightarrow_p \frac{1}{1 - \lambda} \mathcal{I}_0 - \frac{\lambda}{1-\lambda} \mathcal{I}_0 = \mathcal{I}_0 \enspace.
  \end{align*}
  Recall that $\info_n(\htheta_n; \tau_n) = \info_{\post{\tau_n+1}}(\est_n) - \info_{\post{\tau_n+1}}(\est_n)^\top \info_{\post{1}}(\est_n)^{-1} \info_{\post{\tau_n+1}}(\est_n)$, we can derive
  \begin{equation*}
    \frac{n}{\tau_n(n - \tau_n)}\info_n(\htheta_n; \tau_n) \rightarrow_p \frac1\lambda \mathcal{I}_0 - \left(\frac1\lambda - 1\right) \mathcal{I}_0 = \mathcal{I}_0 \enspace.
  \end{equation*}
\end{proof}

To derive the asymptotic distribution of the score $S_{\post{\tau_n+1}}(\htheta_n)$, we will express it as a linear combination of the normalized scores $Z_{\tau_n}(\theta_0) := S_{1:\tau_n}(\theta_0)/\sqrt{\tau_n}$ and $Z_n(\theta_0) := S_{1:n}(\theta_0)/\sqrt{n}$, and then prove its asymptotic normality by the following lemma.
\begin{lemma}\label{lemma:sum_of_two_scores}
  Under Assumption \ref{cond:4}, we have, for every sequence $\tau_n \in \Z_+$ such that $\tau_n / n \rightarrow \lambda \in (0, 1)$,
  \begin{align}\label{eq:biscore_asymp}
    \begin{pmatrix}
    Z_{\tau_n} \sqrt{\tau_n/n} \\
    Z_n
    \end{pmatrix} \rightarrow_d
    \mathcal{N}\left( 0,
      \begin{pmatrix}
        \lambda\mathcal{I}_0 & \lambda\mathcal{I}_0 \\
        \lambda\mathcal{I}_0 & \mathcal{I}_0
        \end{pmatrix}
    \right) \enspace.
  \end{align}
  Moreover, if $\sqrt{n}(\est_n - \theta_0) = \calI_0^{-1} Z_n(\tpar) + o_p(1)$, then
  \begin{align*}
    \sqrt{n}
    \begin{pmatrix}
    \htheta_{\tau_n} - \theta_0 \\
    \htheta_{n} - \theta_0
    \end{pmatrix}
    \rightarrow_d
    \mathcal{N} \left( 0,
      \begin{pmatrix}
        \lambda^{-1}\mathcal{I}_0^{-1} & \mathcal{I}_0^{-1} \\
        \mathcal{I}_0^{-1} & \mathcal{I}_0^{-1}
        \end{pmatrix}
    \right)\enspace.
  \end{align*}
\end{lemma}
\begin{proof}
  According to Cram\'er-Wold device, it is sufficient to show that for any $(a^\top, b^\top) \in \R^{2d}$,
  \begin{align*}
    a^\top \sqrt{\frac{\tau_n}{n}} Z_{\tau_n} + b^\top Z_n \rightarrow_d \mathcal{N}\left( 0, \lambda(a+b)^\top \mathcal{I}_0 (a+b) + (1-\lambda)b^\top \mathcal{I}_0 b \right), \quad \text{as } n \rightarrow \infty \enspace.
  \end{align*}
  We will prove this by the Lindeberg theorem for martingales.
  In fact,
  \begin{align*}
    a^\top \sqrt{\frac{\tau_n}{n}} Z_{\tau_n} + b^\top Z_n
    & = \sum_{k=1}^{\tau_n} (a+b)^\top \frac{M_k}{\sqrt{n}}
      + \sum_{k=\tau_n + 1}^n b^\top \frac{M_k}{\sqrt{n}} \enspace.
  \end{align*}
  Let $\obs_{n,k} = (a+b)^\top M_k$, if $k \in [\tau_n]$; and $\obs_{n,k} = b^\top M_k$, if $k \in \{\tau_n + 1, \dots, n\}$.
  Then $\{\obs_{n,k}, \mathcal{F}_k\}_{k\in\Z}$ is also a martingale difference sequence.
  Additionally,
  \begin{align*}
    \frac1n\sum_{k=1}^n \Ex[\obs_{n,k}^2 | \mathcal{F}_{k-1}] &= \frac1n \sum_{k=1}^{\tau_n} (a+b)^\top \bbE[M_k M_k^\top|\mathcal{F}_{k-1}] (a+b)
    + \frac1n \sum_{k = \tau_n + 1}^n b^\top \bbE[M_k M_k^\top|\mathcal{F}_{k-1}] b \\
    & = \frac{\tau_n}{n} \frac1{\tau_n} \sum_{k=1}^{\tau_n} a^\top \Ex[M_k M_k^\top|\mathcal{F}_{k-1}] (a+2b) + \frac1n \sum_{k = 1}^n b^\top \Ex[M_k M_k^\top|\mathcal{F}_{k-1}] b \\
    & \rightarrow_p \lambda a^\top \mathcal{I}_0 (a+2b) + b^\top \mathcal{I}_0 b
    = \lambda(a+b)^\top \mathcal{I}_0 (a+b) + (1-\lambda)b^\top \mathcal{I}_0 b\enspace,
\end{align*}
and, for any $\varepsilon > 0$,
\begin{align*}
  & \quad \frac1n \sum_{k=1}^n \Ex[ \obs_{n,k}^2 \ind{(|\obs_{n,k}| > \varepsilon\sqrt{n})}|\mathcal{F}_{k-1}] \\
  & = \frac1n \sum_{k=1}^{\tau_n} \Ex\left[ \left((a+b)^\top M_k \right)^2 \ind{(\lvert (a+b)^\top M_k \rvert > \varepsilon \sqrt{n}}) \bigg\vert \mathcal{F}_{k-1} \right] + \frac1n \sum_{k=\tau_n+1}^n \Ex\left[ \left( b^\top M_k \right)^2 \ind{(\lvert b^\top M_k \rvert > \varepsilon \sqrt{n})} \bigg\vert \mathcal{F}_{k-1} \right] \\
  & \rightarrow_p 0 \enspace,
\end{align*}
by Assumption \ref{cond:4b}.
Therefore, the statement \eqref{eq:biscore_asymp} holds by invoking the Lindeberg theorem for martingales.
Moreover,
\begin{align*}
  \sqrt{n}
  \begin{pmatrix}
    \htheta_{\tau_n} - \theta_0 \\
    \htheta_{n} - \theta_0
  \end{pmatrix}
  & =
  \begin{pmatrix}
    \mathcal{I}_0^{-1} \sqrt{\frac{n}{\tau_n}} Z_{\tau_n} + o_p(1) \\
    \mathcal{I}_0^{-1} Z_n + o_p(1)
  \end{pmatrix}
  =
    \begin{pmatrix}
      \mathcal{I}_0^{-1}/\lambda & 0 \\
      0 & \mathcal{I}_0^{-1}
    \end{pmatrix}
    \begin{pmatrix}
      \sqrt{\tau_n/n} Z_{\tau_n} \\
      Z_n
    \end{pmatrix} + o_p(1) \\
  &\rightarrow_d \mathcal{N}\left(0,
        \begin{pmatrix}
          \lambda^{-1}\mathcal{I}_0^{-1} & \mathcal{I}_0^{-1} \\
          \mathcal{I}_0^{-1} & \mathcal{I}_0^{-1}
        \end{pmatrix} \right)\enspace.
\end{align*}
\end{proof}

\begin{proof-of}{\Cref{prop:null}}
  Since $\est_n$ maximizes the log-likelihood function, it must satisfy the first order optimality condition, \ie $S_{1:n}(\est_n) = 0$.
  Then by Assumption \ref{cond:3} and Taylor expansion,
  \begin{align*}
    Z_n(\theta_0) = Z_n(\htheta_n) - \nabla_{\theta} Z_n(\theta_n^*)^\top (\htheta_n - \theta_0) = - \frac1{\sqrt{n}}\nabla_{\theta} Z_n(\theta_n^*)^\top \sqrt{n}(\htheta_n - \theta_0)\enspace,
  \end{align*}
  where $\theta_n^*$ is between $\theta_0$ and $\htheta_n$.
  It follows that $\theta_n^* \rightarrow_p \theta_0$ and
  \begin{align}\label{eq:convergence_info}
    -\frac{1}{\sqrt{n}} \nabla_{\theta} Z_n(\theta^*_n) = - \frac1n \nabla_\theta^2 \ell_n(\theta_n^*) = \mathcal{I}_0 + o_p(1)
  \end{align}
  by a similar argument as in \Cref{lem:null_info}.
  Note that $\sqrt{n}(\htheta_n - \theta_0) = O_p(1)$, we obtain
  \begin{equation}\label{eq:convergence_mle}
    \sqrt{n}(\htheta_n - \theta_0) = \mathcal{I}_0^{-1} Z_n(\theta_0) + o_p(1)\enspace.
  \end{equation}

  We then express the score $S_{\post{\tau_n+1}}$ as a linear combination of the normalized scores $Z_{\tau_n}(\theta_0)$ and $Z_n(\theta_0)$.
  By Lindeberg theorem for martingales~\cite[Chapter 4.5]{van2013time} and Cram\'er-Wold device~\cite{billingsley2008probability}, Assumption \ref{cond:4} implies $Z_n(\theta_0) \rightarrow_d \mathcal{N}(0, \mathcal{I}_0)$, and thus $Z_n(\theta_0) = O_p(1)$ as $n \rightarrow \infty$. It follows that
  \begin{align*}
  \label{eq:expasion_of_score}
    \frac{S_{\post{\tau_n+1}}(\htheta_n)}{\sqrt{n - \tau_n}} &= \frac{S_{\post{\tau_n+1}}(\theta_0)}{\sqrt{n - \tau_n}} + \frac1{\sqrt{n - \tau_n}}\nabla_{\theta} S_{\post{\tau_n+1}}^\top (\theta^*_n)(\htheta_n - \theta_0) \\
    & = \frac{S_{\post{\tau_n+1}}(\theta_0)}{\sqrt{n - \tau_n}} + \frac{\left( \nabla_{\theta} S_{1:n}(\theta^*_n) - \nabla_{\theta} S_{1:\tau_n}(\theta^*_n) \right)^\top}{\sqrt{n(n - \tau_n)}} \sqrt{n}(\htheta_n - \theta_0) \nonumber\\
    &= \frac{S_{\post{\tau_n+1}}(\theta_0)}{\sqrt{n - \tau_n}} + \left[\sqrt{\frac{n}{n - \tau_n}} \frac{Z_{n}(\theta_n^*)}{\sqrt{n}} - \frac{\tau_n}{\sqrt{n(n - \tau_n)}} \frac{Z_{\tau_n}(\theta_n^*)}{\sqrt{\tau_n}} \right]^\top (\calI_0^{-1} Z_n(\tpar) + o_p(1)), \quad \mbox{by } \eqref{eq:convergence_mle} \nonumber\\
    &= \frac{S_{\post{\tau_n+1}}(\theta_0)}{\sqrt{n - \tau_n}} + \left( \frac{\lambda}{\sqrt{1-\lambda}} - \sqrt{\frac{1}{1-\lambda}} \right) \mathcal{I}_0 \mathcal{I}_0^{-1} Z_n(\theta_0) + o_p(1), \quad \mbox{by } \eqref{eq:convergence_info} \nonumber \\
    & = -\sqrt{\frac{\tau_n}{n - \tau_n}} Z_{\tau_n}(\theta_0) + \sqrt{\frac{n}{n - \tau_n}}Z_n(\theta_0) + \frac{\lambda - 1}{\sqrt{1 - \lambda}} Z_n(\theta_0) + o_p(1) \nonumber\\
    & = - \frac{\sqrt{\lambda}}{\sqrt{1 - \lambda}}Z_{\tau_n}(\theta_0) + \frac{\lambda}{\sqrt{1 - \lambda}}Z_n(\theta_0) + o_p(1)\enspace.
  \end{align*}
  Now by Lemma~\ref{lemma:sum_of_two_scores}, we have
  \begin{equation}\label{eq:asn_score}
    \sqrt{\frac{n}{\tau_n(n - \tau_n)}} S_{\post{\tau_n+1}}(\htheta_n)
    \rightarrow_d
    \mathcal{N}\left( 0, \Big[\frac1\lambda \frac{\lambda}{1 - \lambda} - \frac2\lambda \frac{\lambda^2}{1 - \lambda} + \frac1\lambda \frac{\lambda^2}{1 - \lambda}\Big] \mathcal{I}_0 \right)
    =_d \mathcal{N}(0, \mathcal{I}_0) \enspace.
  \end{equation}
  Therefore, by \Cref{lem:null_info} and \eqref{eq:asn_score}, we have $R_n(\tau_n) \rightarrow_d \chi_d^2$ and $R_n(\tau_n, T) \rightarrow_d \chi_{\abs{T}}^2$.
\end{proof-of}

Note that the linear statistic is the maximum of $R_n(\tau)$ over $\tau \in [n-1]$, so we use the Bonferroni correction to compensate for multiple comparisons.
This gives the threshold $H_\text{lin}(\alpha) = q_{\chi_d^2}({\alpha}/{n})$---the upper $({\alpha}/{n})$-quantile of $\chi_d^2$.
Similarly, since the asymptotic distribution of $R_n(\tau, T)$ with $T \in \mathcal{T}_p$ is $\chi_p^2$ and $\abs{\mathcal{T}_p} = \binom{d}{p}$, the Bonferroni correction leads to the threshold $H_p(\alpha) = q_{\chi_p^2}(\alpha/[\binom{d}{p}n(p+1)^2])$,
where $(p+1)^2$ is required to guarantee an asymptotic $\alpha$ level.
In fact, we only need $\sum_{p \in \mathcal{P}} 1/(p+1)^2 < 1$ for controlling the level.
Other corrections are possible, but the former provides small thresholds when the change is sparse.

\begin{corollary}\label{cor:asymptotic_level}
  Under \Cref{asmp:null}, the three tests $\psi_{\text{auto}}, \psi_\lin, \psi_\scan$ are consistent in level with thresholds defined above.
\end{corollary}
\begin{proof}
Let $\bbE_0$ and $\prob_0$ be the expectation and probability distribution under the null hypothesis.
We have
\begin{align*}
  \bbE_0[\psi_\text{lin}(\alpha)] = \bbP_0\left\{\max_{\tau \in [n-1]} R_n(\tau) > H_\lin(\alpha)\right\} \le \sum_{\tau=1}^{n-1} \prob_0(R_n(\tau) > q_{\chi_d^2}(\alpha/n)) \le \sum_{\tau=1}^{n-1} \frac{\alpha}{n} + o(1) = \alpha + o(1) \enspace,
\end{align*}
and
\begin{align*}
    \mathbb{E}_0[\psi_\text{scan}(\alpha)]
    &= \bbP_0\Big( \max_{\tau \in [n-1]} \max_{p \le P} \max_{T \in \mathcal{T}_p} H_p(\alpha)^{-1} R_n(\tau, T) > 1 \Big) \\
    & \le \sum_{\tau = 1}^{n-1} \sum_{p \le P} \sum_{T \in \mathcal{T}_p} \prob_0 \bigg( \frac{ R_n(\tau, T)}{q_{\chi_p^2}\big(\alpha/\big(\binom{d}{p}n(p+1)^2\big)\big)} > 1 \bigg) \\
    & \le \sum_{\tau = 1}^{n-1} \sum_{p \le P} \sum_{T \in \mathcal{T}_p} \frac{\alpha}{\binom{d}{p}n(p+1)^2} + o(1)
    < \sum_{p=1}^\infty \frac{\alpha}{(p+1)^2} + o(1)
    < \alpha + o(1) \enspace.
\end{align*}
For $\alpha = \alpha_l + \alpha_s$, the \emph{autograd-test} has false alarm rate
\begin{align*}
  \bbE_0[\psi(\alpha)] \le \bbE_0[\psi_\text{lin}(\alpha_l)] + \bbE_0[\psi_\text{scan}(\alpha_s)] \le \alpha_l + \alpha_s + o(1) = \alpha + o(1) \enspace.
\end{align*}
Therefore, the three proposed tests are all consistent in level.
\end{proof}

\subsection{Fixed alternative hypothesis}
Under fixed alternative hypothesis, we make the following assumptions.
\begin{assumption}\label{asmp:fixed}
  Let $\obs_1, \dots, \obs_n$ be an independent sample and $\{p_\theta: \theta \in \Theta \subset \bbR^d\}$ be a family of density functions.
  Suppose that there exists $\tau_n \in [n-1]$ such that $\obs_1, \dots,\obs_{\tau_n} \sim p_{\theta_0}$, $\obs_{\tau_n + 1},\dots,\obs_n \sim p_{\theta_1} (\theta_1 \neq \theta_0)$, and $\tau_n / n \rightarrow \lambda \in (0, 1)$.
  Moreover, suppose that the following assumptions hold:
  \begin{enumerate}[label=A'\arabic*,itemsep=1pt,parsep=0pt,topsep=0pt,partopsep=0pt]
    \item \label{cond:1p} : ~
      $F(\theta) := \lambda D_\text{KL}(p_{\theta_0} \Vert p_\theta) + \blambda D_\text{KL}(p_{\theta_1}\Vert p_\theta)$ has a minimizer $\theta^* \in int(\Theta)$, where $\blambda = 1 - \lambda$ and $D_\text{KL}$ is the KL-divergence.
    \item \label{cond:2p}: ~
    $\Theta$ contains an open neighborhood $\Theta^*$ of $\theta^*$ for which
    \begin{enumerate}[label=\ref*{cond:2p}-(\alph*)]
      \item : ~ $\ell(\theta) := \ell(\theta|x) := \log{p_\theta(x)}$ is twice continuously differentiable in $\theta$ almost surely.
      \item : ~ $\nabla_{ijk}^3 \ell(\theta|x)$ exists and satisfies $\abs{\nabla_{ijk}^3 \ell(\theta|x)} \le M_{ijk}(x)$ for $\theta \in \Theta^*$ and $i,j,k \in [d]$ almost surely with $\bbE_{\theta_l} M_{ijk}(\obs) < \infty$ for $l \in \{0,1\}$.
    \end{enumerate}
    \item \label{cond:3p}: ~ $\bbE_{\theta_l}[\nabla_{\theta} \ell(\theta^*)] = \nabla_{\theta} \bbE_{\theta_l}[\ell(\theta)]|_{\theta = \theta^*} = S_l^*$ for $l \in \{0, 1\}$.
    \item \label{cond:4p}: ~ $\bbE_{\theta_l}[-\nabla_{\theta}^2 \ell(\theta^*)] = \calI_l^*$ is positive definite for $l \in \{0, 1\}$.
  \end{enumerate}
\end{assumption}

\begin{proposition}[Fixed alternative hypothesis]\label{prop:fix_alternative}
  Under \Cref{asmp:fixed},
  there exists a sequence of MLE such that $\htheta_n \rightarrow_p \theta^*$ and, for any $\tau_n/n \rightarrow \lambda \in (0, 1)$,
  \begin{align}
    \tfrac{1}{n} R_n(\tau_n) \rightarrow_p (\blambda S_1^*)^\top (\calI^*)^{-1} (\blambda S_1^*)\enspace,
  \end{align}
  where $\calI^* = \blambda \calI_1^* - \blambda \calI_1^* \left(\lambda\calI_0^* + \blambda \calI_1^*\right)^{-1}\blambda \calI_1^*$ is a positive definite matrix.
\end{proposition}
\begin{proof}
  Among all solutions of the likelihood equation $\nabla_{\theta} \ell_n(\theta) = 0$, let $\est_n$ be the one that is closest to $\theta^\ast$ (this is possible since we are proving the existence).
  We firstly prove that $\htheta_n \rightarrow_p \theta^\ast$.
  For $\varepsilon > 0$ sufficiently small, let $B_\varepsilon = \{\theta \in \bbR^d: \norm{\theta - \theta^*} \le \varepsilon\} \subset \Theta^*$ and $\text{bd}(B_\varepsilon)$ be the boundary of $B_\varepsilon$.
  We will show that, for sufficiently small $\varepsilon$,
  \begin{equation}
  \label{eq:target}
    \bbP\left(\ell_n(\theta) < \ell_n(\theta^*), \forall \theta \in \text{bd}(B_\varepsilon)\right) \rightarrow 1 \enspace.
  \end{equation}
  This implies, with probability converging to one, $\ell_n(\theta)$ has a local maximum (also a solution to the likelihood equation) in $B_\varepsilon$, and thus $\htheta_n \in B_\varepsilon$.
  Consequently, $\bbP(\Vert \htheta_n - \theta^\ast \Vert > \varepsilon) \rightarrow 0$.

  To prove \eqref{eq:target}, we write, for any $\theta \in \text{bd}(B_\varepsilon)$, that
  \begin{align*}
    \frac{1}{n} [\ell_n(\theta) - \ell_n(\theta^*)]
    &= \frac{1}{n} (\theta - \theta^*)^\top \nabla_{\theta} \ell_n(\theta^*)
      - \frac{1}{2} (\theta - \theta^*)^\top \left(-\frac{1}{n} \nabla_{\theta}^2 \ell_n(\theta^*) \right) (\theta - \theta^*) \\
      &\quad + \frac{1}{6n} \sum_{i=1}^d \sum_{j=1}^d \sum_{k=1}^d (\theta_i - \theta_i^*)(\theta_j - \theta_j^*)(\theta_k - \theta_k^*) \nabla_{ijk}\ell_n(\bar{\theta}_n) \\
    &=: D_1 + D_2 + D_3\enspace,
  \end{align*}
  where $\bar{\theta}_n \in B_\varepsilon$ satisfies $\norm{\bar{\theta}_n - \theta^*} \le \norm{\theta - \theta^*}$.
  Let us bound $D_1$, $D_2$, and $D_3$ separately.
  Note that, by the law of large numbers,
  \begin{align*}
    D_1 \rightarrow_p &\ (\theta - \theta^*)^\top \left[\lambda \bbE_{\theta_0}[\nabla_{\theta} \ell(\theta^*)] + \blambda \bbE_{\theta_1}[\nabla_{\theta} \ell(\theta^*)]\right] \\
    = &\ (\theta - \theta^*)^\top \nabla_{\theta}
          \left[ \lambda \bbE_{\theta_0}[\ell(\theta)]
      + \blambda \bbE_{\theta_1}[\ell(\theta)]\right] \bigr\rvert_{\theta = \theta^*}, \quad \mbox{by Assumption \ref{cond:3p}} \\
    = &\ -(\theta - \theta^*)^\top \nabla_{\theta} \left[\lambda D_{KL}(p_{\theta_0} \Vert p_\theta) + \blambda D_{KL}(p_{\theta_1}\Vert p_\theta)\right] \bigr\rvert_{\theta = \theta^*} \\
    = &\ 0\enspace,
  \end{align*}
  where the last equality follows from Assumption \ref{cond:1p}.
  Moreover, by Assumption \ref{cond:4p},
  \begin{align*}
    D_2 \rightarrow_p
    - \frac{1}{2} (\theta - \theta^*)^\top
      \left(\lambda \calI_0^* + \blambda \calI_1^* \right)(\theta - \theta^*)
    \le - \frac{1}{2} \lambda_{\min} \varepsilon^2\enspace,
  \end{align*}
  where $\lambda_{\min}$ is the smallest eigenvalue of $\lambda \calI_0^* + \blambda \calI_1^*$.
  If we set $\varepsilon$ small enough such that $\text{bd}(B_\varepsilon) \subset \Theta^*$, then we have, by Assumption \ref{cond:2p},
  \begin{align*}
    \abs{D_3}
    & \le \frac{1}{6n}
          \sum_{ijk} \abs{\theta_i - \theta_i^*} \abs{\theta_j - \theta_j^*} \abs{\theta_k - \theta_k^*}
          \sum_{l=1}^n \abs{\nabla_{ijk} \ell(\bar{\theta}_n|\obs_l)}, \quad \mbox{by triangle inequality} \\
    & \le \frac{1}{6} \varepsilon^3
          \sum_{ijk} \frac{1}{n} \sum_{l=1}^n M_{ijk}(\obs_l), \quad \mbox{by } \abs{\theta_i - \theta_i^*} \le \norm{\theta - \theta^*} = \varepsilon \\
    & \rightarrow_p \frac{\varepsilon^3}{6}
          \sum_{ijk}\left(\lambda \bbE_{\theta_0}[M_{ijk}(\obs)]
          + \blambda \bbE_{\theta_1}[M_{ijk}(\obs)]\right)\enspace.
  \end{align*}

  Hence, for any given $\delta > 0$, any $\varepsilon > 0$ sufficiently small, any $n$ sufficiently large, with probability larger than $1 - \delta$, we have, for all $\theta \in \text{bd}(B_\varepsilon)$,
  \begin{align*}
    \abs{D_1} < \varepsilon^3, \quad D_2 < -\lambda_{\min} \varepsilon^2 / 4, \quad \abs{D_3} \le A\varepsilon^3\enspace,
  \end{align*}
  where $A > 0$ is a constant.
  It follows that,
  \begin{align*}
    D_1 + D_2 + D_3 < \varepsilon^3 + A\varepsilon^3 - \frac{\lambda_{\min}}{4} \varepsilon^2 = \left((A + 1)\varepsilon - \frac{\lambda_{\min}}{4}\right)\varepsilon^2 < 0, \quad \text{if } \varepsilon < \frac{\lambda_{\min}}{4(A+1)} \enspace,
  \end{align*}
  and thus \eqref{eq:target} holds.

  Now, following a similar argument as in \Cref{lem:null_info}, we obtain
  \begin{align*}
    \frac1n S_{\post{\tau_n+1}}(\htheta_n)
    &= \frac1n S_{\post{\tau_n+1}}(\theta^*) + o_p(1) \rightarrow_p \blambda S_1^* \\
    \frac1n \info_n(\htheta_n;\tau_n)
    &= \frac1n \info_n(\theta^*;\tau_n) + o_p(1) \rightarrow_p \blambda \calI_1^*
      - \blambda \calI_1^* \left(\lambda \calI_0^* + \blambda
      \calI_1^*\right)^{-1} \blambda \calI_1^*
    \equiv \calI^*\enspace,
  \end{align*}
  where $\calI^*$ is positive definite since both $\calI_0^*$ and $\calI_1^*$ are positive definite.
  This implies
  \begin{align*}
    \frac{1}{n} R_n(\tau_n)
    &=
    \left(\frac{1}{n} S_{\post{\tau_n+1}}(\htheta_n)\right)^\top
    \left(\frac{1}{n} \info_n (\htheta_n;\tau_n)\right)
    \left(\frac{1}{n} S_{\post{\tau_n+1}}(\htheta_n)\right) \rightarrow_p
    (\blambda S_1^*)^\top (\calI^*)^{-1} (\blambda S_1^*)\enspace.
  \end{align*}
\end{proof}

To show the power consistency of the proposed tests, it suffices to prove $H_{\text{lin}}(\alpha)/n = o(1)$ and $H_{p}(\alpha)/n = o(1)$ for all $p \in [P]$.
For this purpose, we recall a concentration inequality valid for $\chi^2$ distributions introduced in \cite{birge2001alternative}.

\begin{lemma}
  \label{lemma:non-central_chi-square}
  Let $\obs$ be a chi-square random variable with degrees of freedom $d$, that is, $\obs \sim \chi_d^2$. Then, for all $x > 0$,
  \begin{align*}
    \prob\left\{ \obs \ge d + 2\sqrt{dx} + 2x \right\} \le e^{-x} \enspace.
  \end{align*}
\end{lemma}

\begin{corollary}
  Suppose that \Cref{asmp:fixed} is true and $S_1^* \neq 0$, then the three tests $\psi_\text{auto}, \psi_\lin, \psi_\scan$ are consistent in power.
\end{corollary}
\begin{proof}
  According to \Cref{lemma:non-central_chi-square}, we have, for any $\alpha \in (0, 1)$,
  \begin{align*}
    H_\text{lin}(\alpha) = q_{\chi_d^2}(\alpha/n) \le d + 2\sqrt{d \log(n/\alpha)} + 2\log(n/\alpha) \enspace,
  \end{align*}
  and thus $H_\text{lin}(\alpha)/n \rightarrow 0$.
  Recall from \Cref{prop:fix_alternative} that
  \begin{align*}
    \frac{1}{n} R_n(\tau_n)
    \rightarrow_p
    (\blambda S_1^*)^\top (\calI^*)^{-1} (\blambda S_1^*)\enspace.
  \end{align*}
  If $S_1^* \neq 0$, then it follows from the positive definiteness of $\calI^*$ that
  \begin{align*}
    \bbP(\psi_\text{lin}(\alpha) = 1) = \bbP\left(R_\text{lin} > H_\text{lin}(\alpha)\right) \ge \bbP\left(\frac1n R_n(\tau_n) > \frac1n H_\text{lin}(\alpha)\right) \rightarrow 1 \enspace.
  \end{align*}
  Analogously, we get
  \begin{align*}
    H_{p}(\alpha) = q_{\chi_p^2}\left(\alpha/\left(\binom{d}{p}n(p+1)^2\right)\right)
    \le p + 2\Big\{p\log{\Big[ \binom{d}{p} n(p+1)^2 / \alpha \Big]}\Big\}^{1/2}
        + 2\log{\Big[ \binom{d}{p} n(p+1)^2 / \alpha \Big]}\enspace,
  \end{align*}
  which implies $H_p(\alpha) / n \rightarrow 0$.
  Therefore, it follows that $\bbP(\psi_\text{scan}(\alpha) = 1) \rightarrow 1$, and subsequently, $\bbP(\psi_\text{auto}(\alpha) = 1) \rightarrow 1$.
\end{proof}

\subsection{Local alternative hypothesis}
Under local alternative hypothesis, we make the following assumptions.
\begin{assumption}\label{asmp:local}
  Let $\obs_1, \dots, \obs_n$ be an independent sample and $\{p_\theta: \theta \in \Theta \subset \bbR^d\}$ be a family of density functions.
  Suppose that there exists $\tau_n \in [n-1]$ such that $\obs_1,\dots,\obs_{\tau_n} \sim p_{\theta_0}$, $\obs_{\tau_n + 1},\dots,\obs_n \sim p_{\theta_n}$ in which $\theta_n = \theta_0 + hn^{-1/2}$ with $h \neq 0$, and $\tau_n / n \rightarrow \lambda \in (0, 1)$.
  Moreover, suppose that the following assumptions hold:
  \begin{enumerate}[label=A''\arabic*,itemsep=1pt,parsep=0pt,topsep=0pt,partopsep=0pt]
    \item \label{cond:2pp} : ~
    $\Theta$ contains an open neighborhood $\Theta_0$ of $\theta_0$ for which
    \begin{enumerate}[label=\ref*{cond:2pp}-(\alph*)]
      \item : \label{cond:2ppa} ~ $\ell(\theta) := \ell(\theta|x) := \log{p_\theta(x)}$ is twice continuously differentiable in $\theta$ almost surely.
      \item : \label{cond:2ppb} ~ $\nabla_{ijk}^3 \ell(\theta|x)$ exists and satisfies $\abs{\nabla_{ijk}^3 \ell(\theta|x)} \le M_{ijk}(x)$ for $\theta \in \Theta_0$ and $i,j,k \in [d]$ almost surely with $\bbE_{\theta_0} M_{ijk}(\obs) < \infty$.
    \end{enumerate}
    \item : ~ \label{cond:3pp} $\bbE_{\theta_0}[\nabla_{\theta} \ell(\theta_0)] = \nabla_{\theta} \bbE_{\theta_0}[\ell(\theta)]|_{\theta = \theta_0} = S_0$.
    \item : ~ \label{cond:4pp} $\bbE_{\tpar}[\nabla_\theta \ell(\tpar) \nabla_\theta \ell(\tpar)^\top] = \bbE_{\theta_0}[-\nabla_{\theta}^2 \ell(\theta_0)] = \calI_0$ is positive definite.
  \end{enumerate}
\end{assumption}

\begin{proposition}[Local alternative hypothesis]\label{prop:local_alternatives}
  Under \Cref{asmp:local},
  there exists a sequence of MLE $\htheta_n$ such that
  \begin{align}
    \label{eq:consistency_in_info_and_asn_mle}
    \frac{n}{\tau_n(n - \tau_n)}\info_n(\htheta_n; \tau_n) &\rightarrow_p \mathcal{I}_0 \\
    \sqrt{n}(\htheta_n - \theta_0) &\rightarrow_d \calN_d\left(\blambda h, \calI_0^{-1}\right) \label{eq:mle_local}\\
    \sqrt{\frac{n}{\tau_n(n - \tau_n)}} S_{\post{\tau_n+1}}(\htheta_n) &\rightarrow_d \calN_d(\sqrt{\lambda\blambda }\ \calI_0 h, \calI_0)\enspace.
  \label{eq:asn_score_local}
  \end{align}
  In particular,
  \begin{align*}
    R_n(\tau_n) &\rightarrow_d \chi_d^2\big(\lambda \bar \lambda h^\top \calI_0 h\big) \\
    R_n(\tau_n, T) &\rightarrow_d \chi_{|T|}^2\Big(\lambda \bar \lambda [\calI_0 h]_T^\top [\calI_0]_{T,T}^{-1} [\calI_0 h]_T \Big) \enspace.
  \end{align*}
\end{proposition}

\begin{proof}
    In this proof we firstly analyze the behavior of the score statistic under the null hypothesis, then we use Le Cam's third lemma (e.g., \cite{van2000asymptotic}), to attain the asymptotic distribution of the test statistic under local alternatives.

    Under $\prob_0 := \prob_{\theta_0}$, an argument similar to the one in \cref{prop:fix_alternative} implies that there exists a sequence of MLE such that $\htheta_n \rightarrow_p \theta_0$, then \eqref{eq:consistency_in_info_and_asn_mle} directly follows from the proof in \cref{lem:null_info}.
    Furthermore, by Assumption \ref{cond:2ppa} and the mean value theorem, there exists $\bar{\theta}_n$ such that $\Vert \bar{\theta}_n - \theta_0 \Vert \le \Vert \htheta_n - \theta_0 \Vert$, and
    \begin{align*}
      0 = \frac{1}{\sqrt{n}} S_{1:n}(\htheta_n)
        = \frac{1}{\sqrt{n}} S_{1:n}(\theta_0)
        + \frac{1}{n}
        \nabla_{\theta} S_{1:n}(\bar{\theta}_n)\sqrt{n}(\htheta_n - \theta_0)
        \enspace.
    \end{align*}
    Since $\htheta_n \rightarrow_p \theta_0$, we have $\bar \theta_n \rightarrow \theta_0$ and thus, by Assumption \ref{cond:2ppb},
    \begin{align*}
      \frac1n \nabla_\theta S_{1:n}(\bar \theta_n) = \frac1n \nabla_\theta S_{1:n}(\theta_0) + o_p(1) = - \calI_0 + o_p(1) \enspace.
    \end{align*}
    Therefore,
    \begin{align*}
      \sqrt{n}(\htheta_n - \theta_0)
      = \calI_0^{-1} \frac{1}{\sqrt{n}} S_{1:n}(\theta_0) + o_p(1)
      = \frac{1}{\sqrt{n}} \sum_{i=1}^n \bar{S}_i(\theta_0) + o_p(1)\enspace,
    \end{align*}
    where $\bar{S}_i(\theta_0) = \calI_0^{-1} \nabla_{\theta}\ell_i(\theta_0)$.

    We then prove the local asymptotic linearity of the log-likelihood ratio.
    We denote the joint probability measure of $\obs_1,\dots,\obs_n$ under the local alternative as $\prob_{\theta_0,\theta_n}^{(\tau_n)}$.
    It holds that
    \begin{align*}
      \log{\frac{d \prob_{\theta_0, \theta_n}^{(\tau_n)}}{d \prob_{\theta_0}^n}}
      &= \ell_{\post{\tau_n+1}}(\theta_n) - \ell_{\post{\tau_n+1}}(\theta_0) \\
      &= (\theta_n - \tpar)^\top S_{\post{\tau_n+1}}(\tpar) + \frac12 (\theta_n - \tpar)^\top \nabla_\theta S_{\post{\tau_n+1}}(\tpar) (\theta_n - \tpar) + o_p(1) \\
      &= \frac{h^\top}{\sqrt{n}} S_{\post{\tau_n+1}}(\theta_0) + \frac12 h^\top \frac{\nabla_\theta S_{\post{\tau_n+1}}(\tpar)}{n} h + o_p(1)
      = h^\top \frac{1}{\sqrt{n}} S_{\post{\tau_n+1}}(\theta_0) - \frac{\blambda }{2} h^\top  \calI_0 h + o_p(1) \enspace.
    \end{align*}
    For any $a \in \bbR^d$, it follows from the multivariate Central Limit Theorem~\cite{billingsley2008probability} that
    \begin{align*}
      \begin{pmatrix}
        a^\top \sqrt{n}(\htheta_n - \theta_0) \\ \log{\frac{d \prob_{\theta_0, \theta_n}^{(\tau_n)}}{d \prob_{\theta_0}^n}}
      \end{pmatrix}
      &= \frac{1}{\sqrt{n}} \left[\sum_{i=1}^{\tau_n} \begin{pmatrix} a^\top \bar{S}_i(\theta_0) \\ 0 \end{pmatrix} + \sum_{i=\tau_n+1}^n \begin{pmatrix} a^\top \bar{S}_i(\theta_0) \\ h^\top S_i(\theta_0) \end{pmatrix}\right] - \begin{pmatrix} 0 \\ \frac{\sigma^2}{2} \end{pmatrix} + o_p(1) \\
      &\rightarrow_d \calN_2\left(\begin{pmatrix} 0 \\ -\sigma^2/2 \end{pmatrix}, \begin{pmatrix} a^\top \calI_0^{-1} a & \blambda a^\top h \\ \blambda a^\top h & \sigma^2\end{pmatrix}\right) \enspace,
    \end{align*}
    where $\sigma^2 := \blambda h^\top \calI_0 h$.
    Hence, the assumptions of Le Cam's third lemma are fulfilled, and we conclude that, under $\prob_{\theta_0, \theta_n}^{(\tau_n)}$,
    \begin{align*}
      a^\top \sqrt{n}(\htheta_n - \theta_0) \rightarrow_d \calN\left(\blambda a^\top h, a^\top \calI_0^{-1} a\right) \enspace.
    \end{align*}
    By the Cram\'er-Wold device, the statement \eqref{eq:mle_local} holds.

    Notice that, under $\prob_{\theta_0}$,
    \begin{align*}
      \frac{1}{\sqrt{n}} S_{\post{\tau_n+1}}(\htheta_n) &= \frac{1}{\sqrt{n}} S_{\post{\tau_n+1}}(\theta_0) - \blambda \calI_0 \sqrt{n}(\htheta_n - \theta_0) + o_p(1) \\
      &= \frac{1}{\sqrt{n}} \left[\sum_{i=1}^{\tau_n} -\blambda S_i(\theta_0) + \sum_{i=\tau_n+1}^n \lambda S_i(\theta_0)\right] + o_p(1)\enspace.
    \end{align*}
    An analogous argument gives, under $\bbP_{\tpar, \theta_n}^{(\tau_n)}$,
    \begin{align*}
      \frac{1}{\sqrt{n}} S_{\post{\tau_n+1}}(\htheta_n) \rightarrow_d \calN_d(\lambda\blambda \calI_0 h, \lambda\blambda \calI_0) \enspace,
    \end{align*}
    which yields \eqref{eq:asn_score_local}.
    Now, the asymptotic distributions of $R_n(\tau_n)$ and $R_n(\tau_n, T)$ follows immediately from the continuous mapping theorem.
\end{proof}

\subsection{Approximation in the scan statistic}
Recall that, in the computation of the scan statistic, we approximate the maximizer of $\max_{T \in \calT_p} R_n(\tau, T)$ by the indices of the largest $p$ components in $v(\tau) := S_{\tau+1:n}(\est_n)^\top \text{diag}(\info_n(\est_n; \tau))^{-1} S_{\tau+1:n}(\est_n)$.
The next lemma verifies that this approximation is accurate when the difference between the largest eigenvalue and the smallest eigenvalue of $\info_n(\est_n; \tau)^{-1}$ is small compared to $\lVert S_{\tau+1:n}(\est_n) \rVert^2$.
\begin{lemma}
  Let $\alpha \in \bbR^d$, and $A \in \bbR^{d \times d}$ be a symmetric positive definite matrix.
  Consider the optimization problem:
  \begin{align*}
    T^* = \argmax_{T\subset [d], |T| = p} f(T) = \argmax_{T\subset [d], |T| = p} \alpha_T^\top [A_{T, T}]^{-1} \alpha_T, \quad p \in [d] \enspace.
  \end{align*}
  Let $0 < \lambda_1(A) \le \dots \le \lambda_d(A)$ be the eigenvalues of $A$, and $\hat T$ be the indices of the largest $p$ components in $\text{\emph{diag}}(A)^{-1}\alpha^{\odot 2}$, where $\alpha^{\odot 2}$ is the element-wise power.
  Then we have $\lvert f(T^*) - f(\hat T) \rvert \le 2[\lambda_1(A)^{-1} - \lambda_d(A)^{-1}] \norm{\alpha}^2$.
\end{lemma}

\begin{proof}
  Define $g(T) := \alpha_T^\top \text{diag}(A_{T, T})^{-1} \alpha_T$.
  According to the definition of $\hat T$, we have, for any $|T| = p$, $g(T) \le g(\hat T)$.
  In particular, we have $g(T^*) \le g(\hat T)$.
  This implies that
  \begin{align*}
    0 \le f(T^*) - f(\hat T) \le f(T^*) - g(T^*) + g(\hat T) - f(\hat T)\enspace,
  \end{align*}
  and thus it suffices to bound $\abs{f(T) - g(T)}$ for every $|T| = p$.

  On the one hand, note that
  \begin{align*}
    f(T) - g(T) = \alpha_T^\top A_{T, T}^{-1} \alpha_T - \alpha_T^\top (\text{diag}(A_{T, T}))^{-1} \alpha_T
    \le \lambda_p(A_{T, T}^{-1}) \norm{\alpha_T}^2 - a_{\max}^{-1} \norm{\alpha_T}^2 \enspace,
  \end{align*}
  where $a_{\max} := \max_{i \in [d]} a_{ii}$.
  By the Courant-Fischer-Weyl min-max principle, we have $0 < \lambda_1(A) \le \lambda_1(A_{T, T})$,
  which implies $\lambda_p(A_{T, T}^{-1}) = \lambda_1(A_{T, T})^{-1} \le \lambda_1(A)^{-1}$.
  Moreover, since $0 < \lambda_1(A) \le a_{\max} \le \lambda_d(A)$, we have $a_{\max}^{-1} \ge \lambda_d(A)^{-1}$.
  It follows that
  \begin{align*}
    f(T) - g(T) \le [\lambda_1(A)^{-1} - \lambda_d(A)^{-1}] \norm{\alpha}^2 \enspace.
  \end{align*}
  On the other hand, we can obtain, similarly,
  \begin{align*}
    g(T) - f(T) \le [a_{\min}^{-1} - \lambda_1(A_{T, T}^{-1})] \norm{\alpha}^2 \le [\lambda_1(A)^{-1} - \lambda_{d}(A)^{-1}] \norm{\alpha}^2
  \end{align*}
  with $a_{\min} := \min_{i \in [d]} a_{ii}$.
  Therefore, we have
  \begin{align*}
    0 \le f(T^*) - f(\hat T) \le 2[\lambda_1(A)^{-1} - \lambda_d(A)^{-1}] \norm{\alpha}^2 \enspace.
  \end{align*}
\end{proof}

%% file: sections/experiments.tex
\begin{figure}
  \centering
  \includegraphics[width=0.4\linewidth]{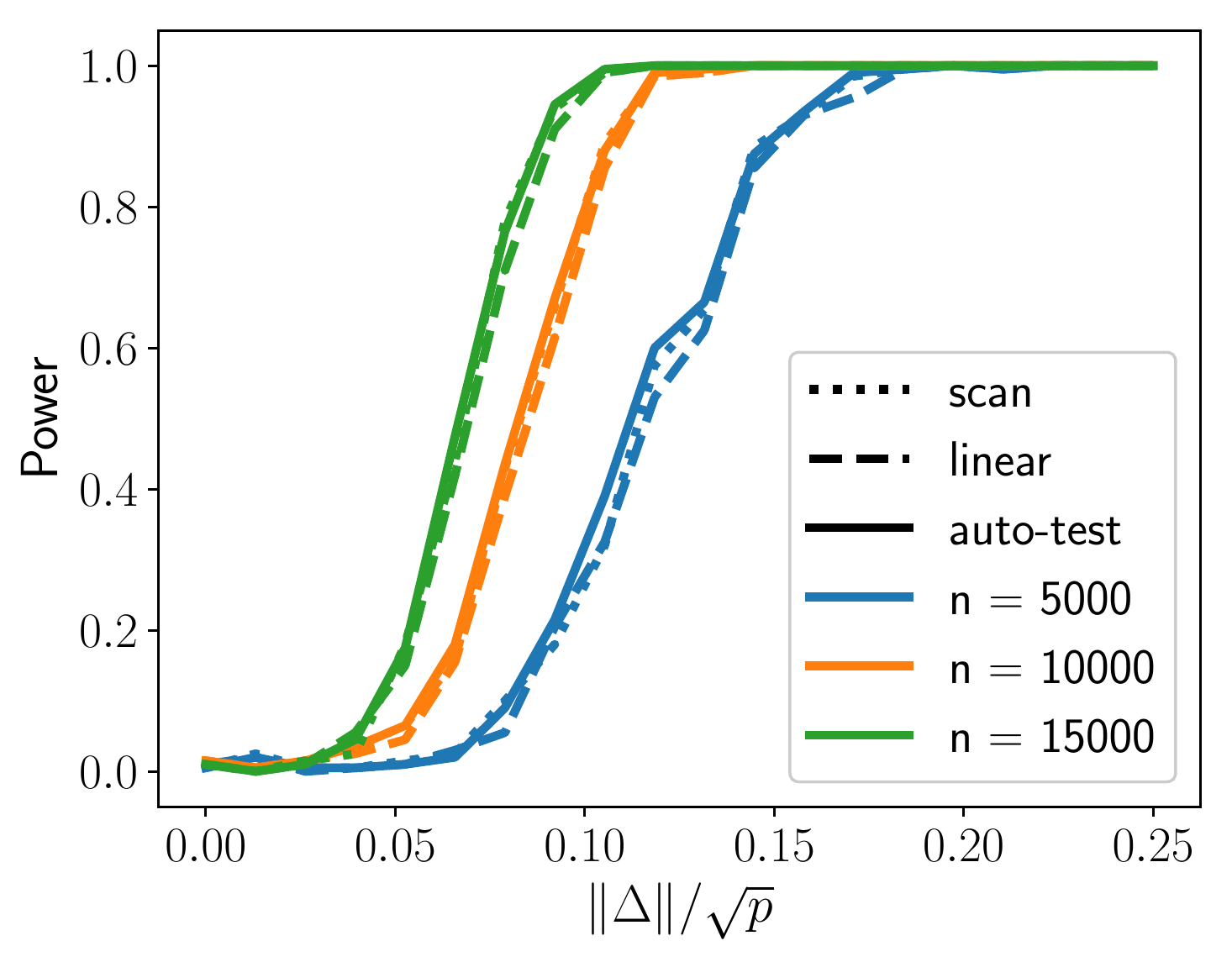}
  \includegraphics[width=0.4\linewidth]{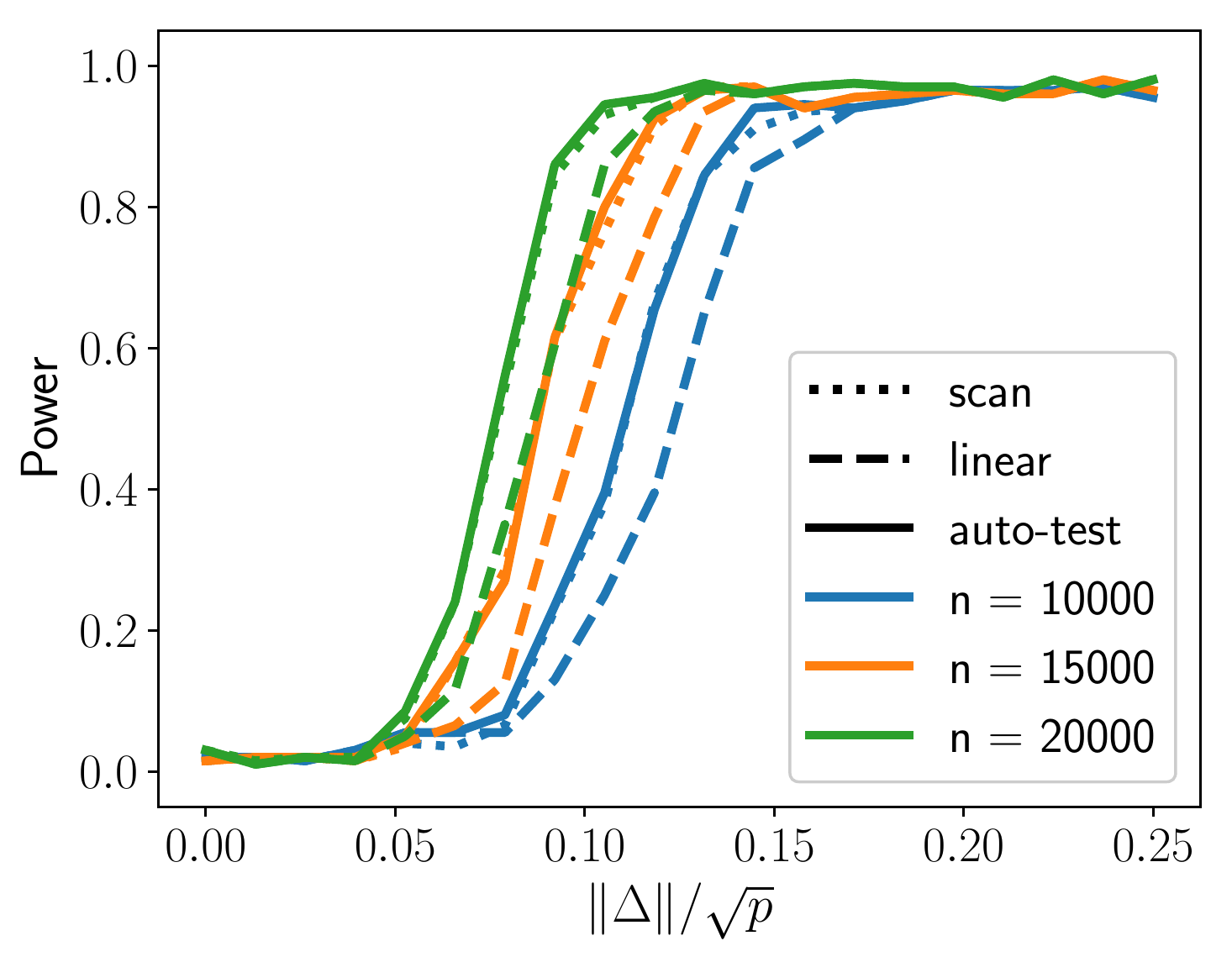}
  \caption{Power versus magnitude of change for HMMs with $N$ hidden states (left: $N=3$; right: $N=7$).}
  \label{fig:hmm}
\end{figure}

In this section, we give additional experimental results investigating and comparing the performance of the linear test, the scan test and the \emph{auto-test} on synthetic data.
We use three different types of lines to represent three tests, and different colors to indicate different sample sizes.
Note that all the statistics are computed only for $\tau \in [n/10, 9n/10]$ to prevent encountering ill-conditioned Fisher information matrix.

\textbf{Hidden Markov model.}
We consider HMMs with $N \in \set{3, 7}$ hidden states and normal emission distribution.
The transition matrix is sampled in the following way: each row (the distribution of next state conditioning on current state) is the sum of vector $(2N)^{-1} \mathbf{1}_N$ and a Dirichlet sample with concentration parameters $0.5 \mathbf{1}_N$, where $\mathbf{1}_N$ is an all one vector of length $N$.
All entries in the resulting vector are positive and sum to one.
Given the state $k \in \{0, \dots, N-1\}$, the emission distribution has mean $k$ and standard deviation $0.01 + 0.09k/(N-1)$ so that they are evenly distributed within $[0.01, 0.1]$.
Since each row of the transition matrix must sum to one, we only view entries in the first $N - 1$ columns as transition parameters.
The post-change transition matrix is obtained by subtracting $\delta$ from the $(1,1)$ entry and adding $\delta$ to the $(1,N)$ entry.

Results are shown in \cref{fig:hmm}.
When $N=3$, three tests have almost identical performance.
When $N=7$, the change becomes sparser, and subsequently, the scan test and the \emph{auto-test} outperform the linear test.
In both cases, the three tests show consistent behavior as the sample size increases.

\textbf{Time series model.}
We then consider two autoregressive--moving-average models---ARMA$(3,2)$ and ARMA$(6,5)$.
For the resulting time series to be stationary, we need to ensure that the polynomial induced by AR coefficients has roots within $(-1,1)$.
We take the following procedure:
we firstly sample $p_0 \in \set{3, 6}$ values that are larger than 1, say $\lambda_1,\dots,\lambda_{p_0}$, then use the coefficients of the polynomial $f_0(x) = \prod_{i=1}^{p_0}(x-\lambda_i^{-1})$ as AR coefficients;
MA coefficients are obtained similarly.
Furthermore, the post-change AR coefficients are created by adding $\delta$ to those $p_0$ values and extracting the coefficients from $f_1(x) = \prod_{i=1}^{p_0}(x-(\lambda_i + \delta)^{-1})$.
The error terms follow a normal distribution with mean 0 and standard deviation 0.1.
Note that for ARMA models we do not have exact control of $\norm{\Delta} / \sqrt{p}$, so readers need to be careful about the range of $x$-axis in \cref{fig:arma}.

As demonstrated in \cref{fig:arma}, the scan test works fairly well for these two ARMA models.
However, the linear test and the \emph{auto-test} have extremely high false alarm rate.
This problem gets more severe as the sample size increases, and hence is not due to the lack of accuracy of the maximum likelihood estimator.

\textbf{Restricted screening components.}
To investigate the high false alarm rate problem in \cref{fig:arma}, we consider the same two ARMA models with the restriction that we only detect changes in the AR coefficients.
As presented in \cref{fig:arma_idx}, all three tests are now consistent in level,
and the linear test and the \emph{auto-test} are slightly more powerful than the scan test.
This suggests that this problem is caused by the non-homogeneity of model parameters.
Indeed, in the experiments in \cref{fig:arma}, the derivatives \wrt AR coefficients are significantly larger than the ones \wrt MA coefficients.
This results in ill-conditioned information matrix and subsequent unstable computation of the linear statistic.
On the contrary, the scan statistic only inverts the submatrix of size $p\times p$, whose condition number is much smaller.
In fact, the parameters selected by the scan statistic are all AR coefficients in our experiments.
Therefore, the scan statistic can produce reasonable results even if the parameters are heterogeneous.
We note that in such situations we can select a small (or even zero) significance level for the linear part in the \emph{auto-test} to obtain reasonable results.

\begin{figure}[t]
  \centering
  \begin{minipage}{0.47\textwidth}
    \centering
    \includegraphics[width=0.49\linewidth]{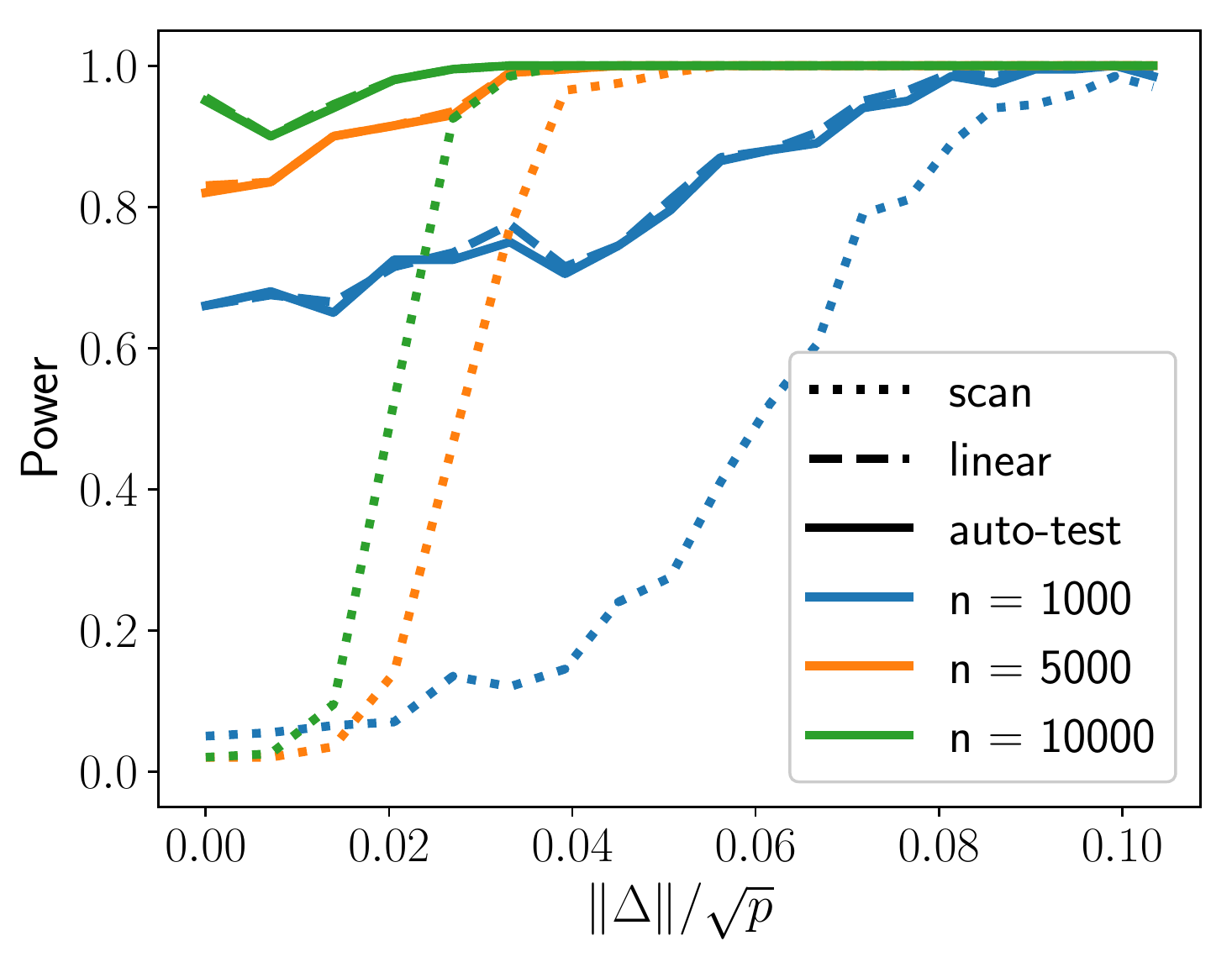}
    \includegraphics[width=0.49\linewidth]{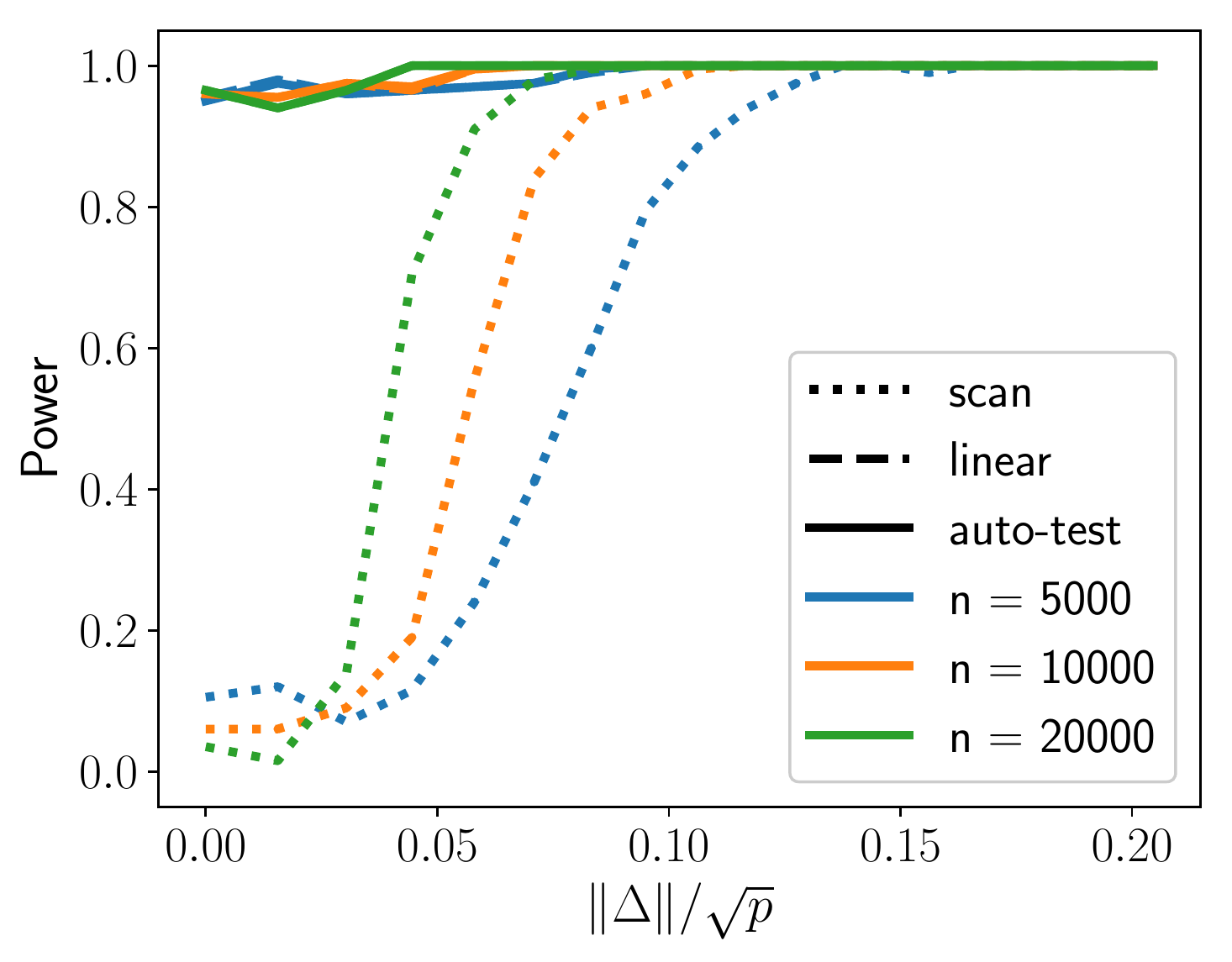}
    \caption{Power versus magnitude of change for ARMA$(3,2)$ (left) and ARMA$(6,5)$ (right). \vspace{0.4cm}}
    \label{fig:arma}
  \end{minipage}
  \hspace{0.5cm}
  \begin{minipage}{0.47\textwidth}
    \centering
    \includegraphics[width=0.49\linewidth]{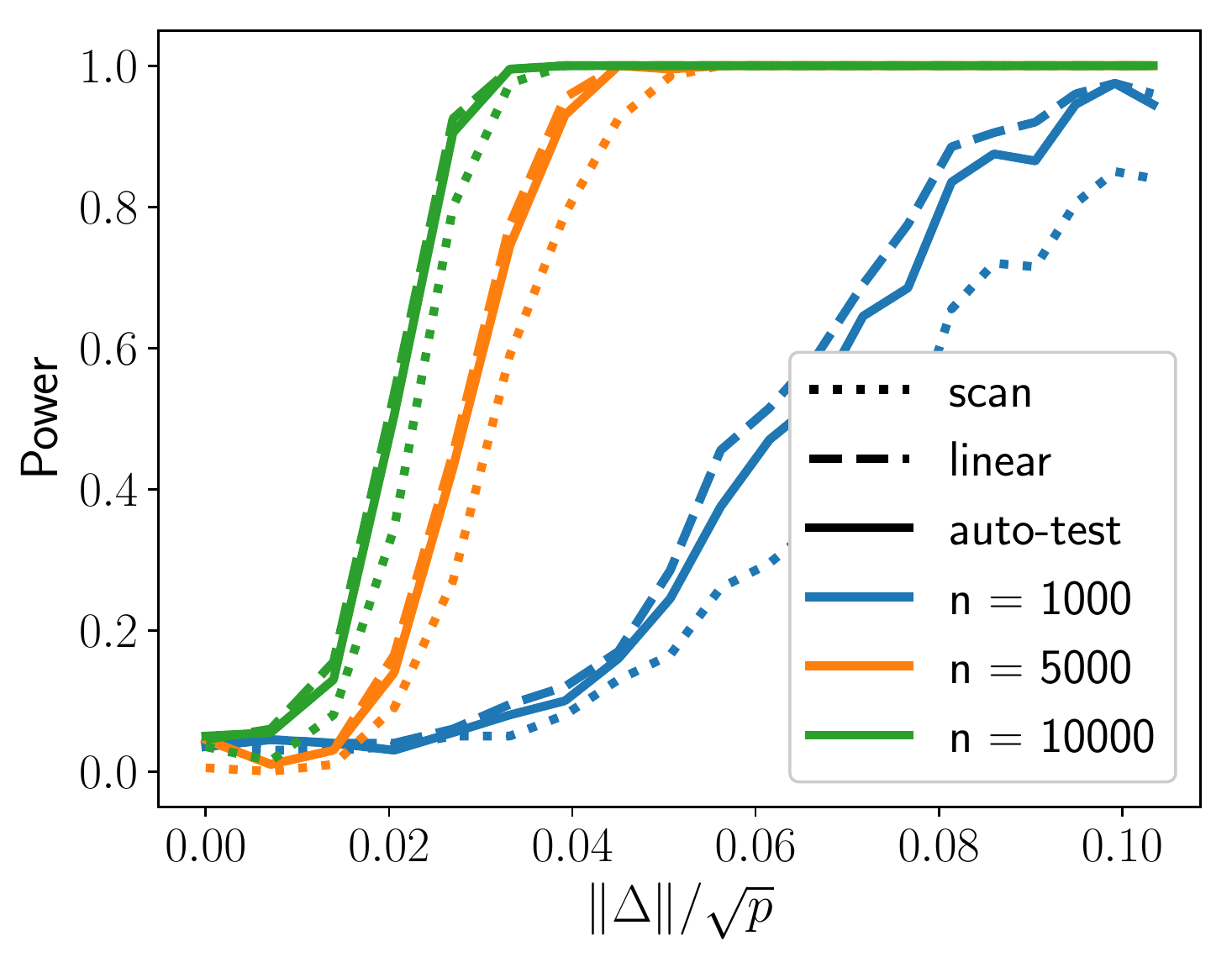}
    \includegraphics[width=0.49\linewidth]{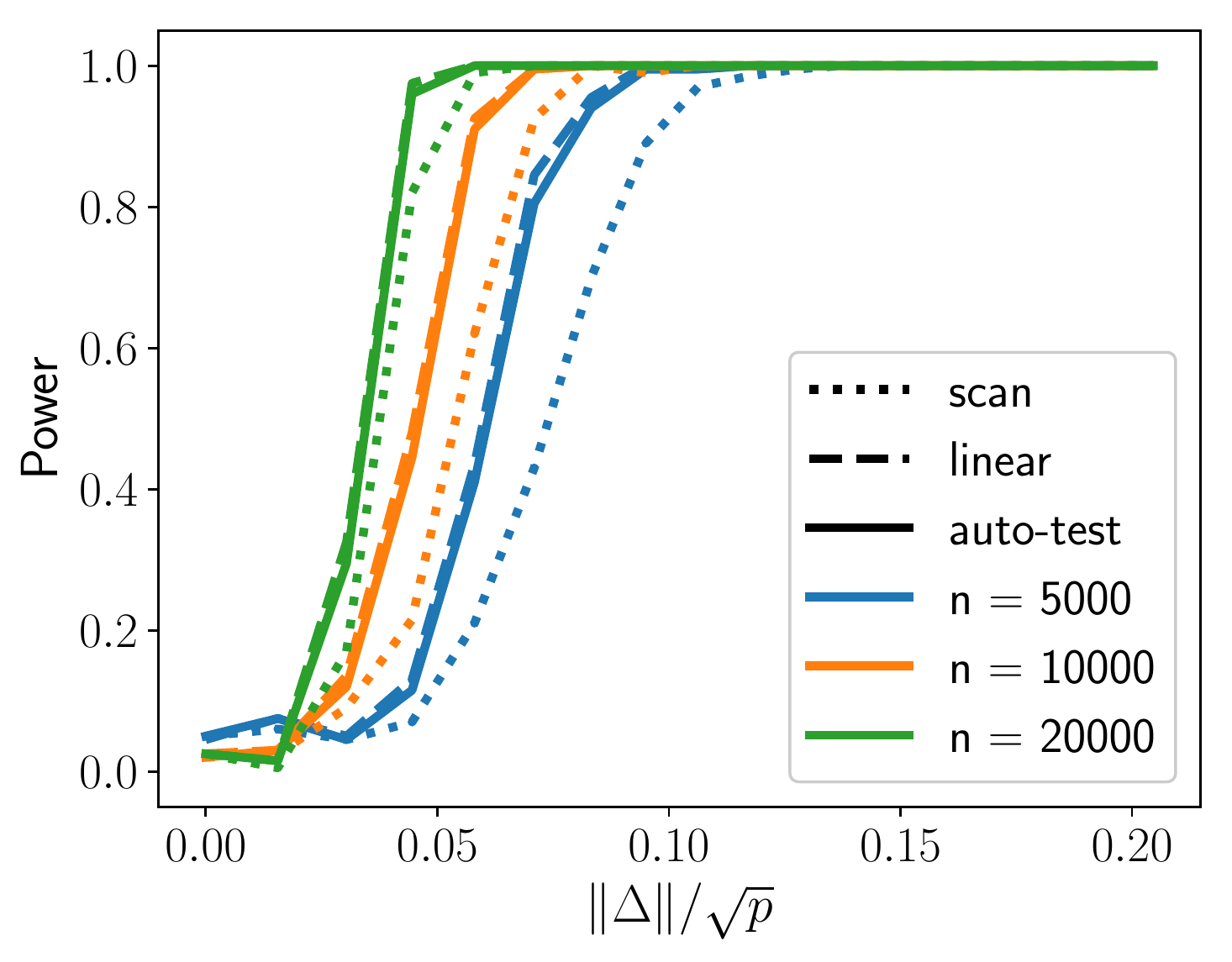}
    \caption{Power versus magnitude of change for ARMA models with restricted components (left: ARMA$(3,2)$; right: ARMA$(6,5)$).}
    \label{fig:arma_idx}
  \end{minipage}
\end{figure}